\documentclass{article}

\usepackage[final]{neurips_2025}

\usepackage[utf8]{inputenc} 
\usepackage[T1]{fontenc}    
\usepackage{hyperref}       
\usepackage{url}            
\usepackage{booktabs}       
\usepackage{amsfonts}       
\usepackage{nicefrac}       
\usepackage{microtype}      
\usepackage{xcolor}         
\usepackage{mathtools}
\usepackage{algorithm}
\usepackage{smile}
\usepackage{enumitem}
\usepackage{makecell}

\DeclareMathOperator{\FG}{FG}
\DeclareMathOperator{\KL}{KL}

\DeclareMathOperator{\LS}{LS}
\DeclareMathOperator{\dc}{dc}
\DeclareMathOperator{\Regret}{Regret}
\newcommand{\FGTSVA}{\texttt{FGTS-VA}}

\allowdisplaybreaks
\ifdefined\final
\usepackage[disable]{todonotes}
\else
\usepackage[textsize=tiny]{todonotes}
\fi
\title{Variance-Aware Feel-Good Thompson Sampling for Contextual Bandits}

\usepackage{colortbl}
\definecolor{LightCyan}{rgb}{0.4, 0.5, 1}

\author{%
  Xuheng Li\\
  Department of Computer Science\\
  University of California, Los Angeles\\
  California, 90095 \\
  \texttt{xuheng.li@cs.ucla.edu} \\
  \And
  Quanquan Gu \\
  Department of Computer Science\\
  University of California, Los Angeles\\
  California, 90095 \\
  \texttt{qgu@cs.ucla.edu} \\
}

\begin{document}

\maketitle

\begin{abstract}
Variance-dependent regret bounds have received increasing attention in recent studies on contextual bandits. However, most of these studies are focused on upper confidence bound (UCB)-based bandit algorithms, while sampling based bandit algorithms such as Thompson sampling are still understudied. The only exception is the \texttt{LinVDTS} algorithm \citep{xu2023noise}, which is limited to linear reward function and its regret bound is not optimal with respect to the model dimension.
In this paper, we present \FGTSVA, a variance-aware Thompson Sampling algorithm for contextual bandits with general reward function with optimal regret bound.
At the core of our analysis is an extension of the decoupling coefficient, a technique commonly used in the analysis of Feel-good Thompson sampling (FGTS) that reflects the complexity of the model space.
With the new decoupling coefficient denoted by $\dc$, \FGTSVA~achieves the regret of $\tilde{\mathcal{O}}\big(\sqrt{\dc\cdot\log|\cF|\sum_{t=1}^T\sigma_t^2}+\dc\big)$, where $|\cF|$ is the size of the model space, $T$ is the total number of rounds, and $\sigma_t^2$ is the subgaussian norm of the noise (e.g., variance when the noise is  Gaussian) at round $t$. In the setting of contextual linear bandits, the regret bound of \FGTSVA~matches that of UCB-based algorithms using weighted linear regression \citep{zhou2022computationally}.
\end{abstract}

\section{Introduction}

The contextual bandit \citep{langford2007epoch} is a pivotal setting in interactive decision making, and is an important generalization of the multi-armed bandit that incorporates context-dependent reward functions.
However, the standard contextual bandit setting is unable to account for heterogeneous and context-dependent noise of the observed rewards, which can have significant impacts on the performance of algorithms \citep{auer2002using}.
Additionally, algorithms designed for standard contextual bandits \citep{abbasi2011improved, chu2011contextual} are usually incompatible with potentially benign environments, despite achieving the minimax regret bound in the worst case.
For example, these algorithms have a worst-case regret bound of $\tilde\cO(d\sqrt T)$ for linear contextual bandits, where $T$ is the number of steps; however, if the reward is deterministic, then the algorithm based on simple exploration can achieve the regret of $\tilde\cO(d)$.
To fill this gap, a number of approaches \citep{zhou2021nearly, zhang2021improved, zhou2022computationally, kim2022improved, xu2023noise, zhao2023variance, jia2024does} have been developed to account for the heterogeneous magnitudes of the noise, and regret bounds are established depending on $\sigma_t^2$, i.e., the noise variance in step $t$.
Most of these approaches are based on the upper confidence bound (UCB).
Notably, \citet{zhou2022computationally} and \citet{zhao2023variance} established the nearly optimal regret bound of $\tilde\cO\big(d\sqrt{\sum_{t=1}^T\sigma_t^2}+d\big)$ in linear bandits, which degenerates to $\tilde\cO(d)$ in the deterministic case, under the settings where $\sigma_t^2$ are known and agnostic to the agent, respectively.

Thompson sampling (TS) \citep{thompson1933likelihood} is another technique that facilitates exploration of the action space apart from UCB-based methods.
In TS-based algorithms, an estimation of the reward function is sampled from the posterior distribution instead of being deterministically constructed.
TS-based algorithms have displayed better efficiency than UCB-based algorithms empirically \citep{chapelle2011empirical, osband2017posterior}, and have been extensively studied for both multi-armed bandits \citep{agrawal2012analysis, kaufmann2012thompson, agrawal2017near, jin2021mots} and linear bandits \citep{agrawal2013thompson}.
However, a minimax-optimal frequentist regret bound of standard TS has been lacking, and \citet{zhang2022feel} constructed an instance such that standard TS is suboptimal.
To resolve this issue, \citet{zhang2022feel} proposed a new variant of TS called the \textbf{Feel-Good Thompson sampling (FGTS)}, and theoretically justifies that the frequentist regret of FGTS is $\tilde\cO(d\sqrt T)$, which is minimax optimal and similar to UCB-based algorithms (e.g., OFUL \citep{abbasi2011improved} and LinUCB \citep{li2010contextual}).

Despite the success the success of TS-based algorithms in the vanilla contextual bandit setting, there is a scarcity of results concerning variance-aware contextual bandits.
A notable exception is \citet{xu2023noise}, which proposed a variant of TS with weighted ridge regression.
However, this result is restricted to linear contextual bandits, and the regret is suboptimal in the model dimension $d$, which is an issue shared by TS-based algorithms designed for vanilla contextual bandits \citep{agrawal2013thompson, abeille2017linear}.
Therefore, the following open question arises:
\begin{quote}
\it Is it possible to design a FGTS-based algorithm for contextual bandits whose regret is both optimal in $d$ and variance-dependent, similar to UCB-based algorithms?
\end{quote}

In this paper, we answer this question affirmatively with the first variance-aware algorithm based on FGTS and new analysis techniques.
We summarize our contributions as follows, with a comparison of our algorithm against related algorithms shown in Table \ref{tab:comparison}:

\begin{itemize}[leftmargin=0.5cm]
\item[1.] We propose an FGTS-based algorithm called \FGTSVA~for the setting of variance-dependent contextual bandits, which is applicable to the general reward function class. Compared with the standard FGTS algorithm in \citet{zhang2022feel}, the posterior distribution of \FGTSVA~adopts not only variance-dependent weights applied to the log-likelihoods but also a new \textbf{feel-good exploration term}. When reduced to standard contextual bandits, \FGTSVA~is the first FGTS-based algorithm that does not require knowledge of the horizon $T$.
\item[2.] In the analysis of \FGTSVA, we propose the \textbf{generalized decoupling coefficient}, which is a novel extension of the standard decoupling coefficient commonly used in the analysis of FGTS. We relate the generalized decoupling coefficient to other complexity measures by showing that (i) the generalized decoupling coefficient is $\tilde\cO(d)$ for linear contextual bandits, and that (ii) it is bounded by the generalized Eluder dimension for the general reward function class.
\item[3.] Equipped with the generalized decoupling coefficient (denoted as $\dc$), we show that the regret bound of \FGTSVA~is $\cO\big(\sqrt{(1+\sum_{t=1}^T\sigma_t^2)\dc\log|\cF|}+\dc\big)$ in expectation, where $|\cF|$ is the cardinality of the function class. For linear contextual bandits, \FGTSVA~enjoys the nearly optimal regret of $\tilde\cO\big(d\sqrt{\sum_{t=1}^T\sigma_t^2}+d\big)$, similar to UCB-based algorithms \citep{zhou2022computationally, zhao2023variance}. When restricted to the deterministic case, the regret of \FGTSVA~is $\tilde\cO(\dc)$, matching the lower bound given by \citet{xu2023noise}.
\end{itemize}

\noindent\textbf{Notation.}
We use $\ind[\cdot]$ to denote the indicator function.
We use $\KL(\cdot||\cdot)$ to denote the KL-divergence of two distributions.
We use standard asymptotic notations $\cO(\cdot)$, $\Omega(\cdot)$, and $\Theta(\cdot)$, with $\tilde\cO(\cdot)$, $\tilde\Omega(\cdot)$, and $\tilde\Theta(\cdot)$ hiding logarithmic factors;
$f(\cdot)\lesssim g(\cdot)$ means $f(\cdot)=\cO(g(\cdot))$.
We use non-boldface letters to denote scalars, boldface lower-case letters to denote vectors, and boldface upper-case letters to denote matrices.
We use $\langle\cdot, \cdot\rangle$ to denote the inner product, i.e., for vectors $\ab$ and $\bbb$, define $\langle\ab, \bbb\rangle=\ab^\top\bbb$; for matrices $\Ab$ and $\Bb$, define $\langle\Ab, \Bb\rangle=\tr(\Ab\Bb^\top)$.
For a vector $\vb$ and a positive semi-definite (PSD) matrix $\Mb$, let $\|\vb\|_{\Mb}=\sqrt{\vb^\top\Mb\vb}$.
For a positive integer $n$, let $[n]$ denote the set of $\{1, 2, \dots, n\}$.

\begin{table}[ht]
    \centering\small
    \caption{Comparison of variance-aware algorithms for bandits. We compare the regret under the setting of both stochastic and deterministic linear bandits, where $d$ is the model dimension, $T$ is the number of steps, and $\Lambda=\sum_{t=1}^T\sigma_t^2$ is the sum of variances. The last column stands for whether the variance is revealed to the learning agent at step $t$.}
    \begin{tabular}{ccccc}
    \toprule
    Algorithm & Technique & Regret (General) & Regret (Deterministic) & $\sigma_t^2$ \\
    \midrule
    \makecell{Weighted OFUL+\\\tiny{\citep{zhou2022computationally}}} & UCB & $\tilde\cO(d\sqrt{\Lambda}+d)$ & $\tilde\cO(d)$ & Known\\
    \makecell{\texttt{SAVE}\\\tiny{\citep{zhao2023variance}}} & UCB & $\tilde\cO(d\sqrt\Lambda+d)$ & $\tilde\cO(d)$ & Unknown\\
    \makecell{\texttt{LinVDTS}\\\tiny{\citep{xu2023noise}}} & TS & $\tilde\cO(d^{1.5}\sqrt\Lambda+d^{1.5})$ & $\tilde\cO(d^{1.5})$ & Unknown\\
    \makecell{\texttt{FGTS}\\\tiny{\citep{zhang2022feel}}} & TS & $\tilde\cO(d\sqrt T)$ & $\tilde\cO(d\sqrt T)$ & NA\\
    \rowcolor{LightCyan!20!} \makecell{~~~~~~~~~\FGTSVA~~~~~~~~~\\\tiny{(This work)}} & TS & $\tilde\cO(d\sqrt\Lambda+d)$ & $\tilde\cO(d)$ & Known\\
    \bottomrule
    \end{tabular}
    \label{tab:comparison}
\end{table}

\section{Related Work}

\noindent\textbf{Variance-aware algorithms.}
\citet{audibert2009exploration} proposed the first algorithm that utilized the variance information through empirical estimates of the variance, with a line of works based on similar techniques for various settings \citep{hazan2011better,  wei2018more, ito2021parameter, ito2023best}.
\citet{mukherjee2018efficient} used the variance estimates to characterize confidence intervals and to perform arm elimination.

A recent line of works study variance-dependent algorithms for bandits with function approximation.
For the case of known variances, a line of works have utilized the variances in weighted ridge regression in linear bandits \citep{zhou2021nearly, zhou2022computationally}.
For the case of unknown variances, \citet{zhang2021improved} and \citet{kim2022improved} constructed variance-dependent confidence sets, and \citet{zhao2023variance} designed a SupLin-type algorithm and proposed the idea of classifying samples into different variance levels.
\citet{di2023variance} used the idea of \citet{zhao2023variance} to develop a variance-aware algorithm for dueling bandits.
Notably, \citet{zhou2022computationally} and \citet{zhao2023variance} managed to derive the nearly optimal regret bounds of $\tilde\cO\big(d\sqrt{\sum_{t=1}^T\sigma_t^2}+d\big)$ for the two cases, respectively.
For contextual bandits with the general function class, \citet{wei2020taking} focused on the case where the action space is small and derived the regret bound related to the total estimation error.
A recent work \citep{jia2024does} studied two cases that are referred to the weak adversary and the strong adversary, depending on whether the actions of the agent can affect variances. This work managed to establish a regret bound of $\tilde\cO\big(d_{\mathrm{elu}}\sqrt{\sum_{t=1}^T\sigma_t^2}+d_{\mathrm{elu}}\big)$ for the strong adversary, where $d_{\mathrm{elu}}$ stands for the Eluder dimension.
However, the analysis of the weak adversary setting, which is more closely related to the setting of our work, is restricted to the case of finite action space.

Another line of works derived second-order bounds for Markov decision processes. \citet{wang2024more} derived the first second-order bound for distributional RL.
Built on the pivotal triangle inequality in \citet{wang2024more}, \citet{wang2024model} proved that OMLE algorithm \citep{liu2023optimistic} without weighted regression enjoys the second-order regret bound.

\noindent\textbf{Feel-Good Thompson sampling (FGTS).}
\citet{zhang2021improved} first proposed FGTS, and achieved the minimax-optimal regret bound for linear contextual bandits by virtue of the feel-good exploration term in the posterior distribution.
\citet{fan2023power} extended the technique to function approximation of the policy space and applied the algorithm to a number of variants of linear contextual bandits.
A recent work by \citet{li2024feel} derived an algorithm for contextual dueling bandits based on FGTS that is efficient both theoretically and empirically.
By replacing the feel-good exploration term with the value function in the first step, \citet{agarwal2022model, dann2021provably} applied FGTS to model-based RL and model-free RL, respectively, with the technique of \citet{agarwal2022model} more similar to \citet{zhang2022feel}, and the technique of \citet{dann2021provably} motivating of our algorithm.

\noindent\textbf{Variance-aware Thompson sampling algorithms.}
\citet{saha2023only} proposed a variance-aware TS-based algorithm for multi-armed bandits.
\citet{xu2023noise} proposed \texttt{LinVDTS}, which is the only algorithm for linear bandits based on Thompson sampling.
The posterior distribution in \texttt{LinVDTS} is Gaussian with the mean estimated with the weighted ridge regression in \citet{zhou2022computationally}.
However, the regret bound of \texttt{LinVDTS} is $\tilde\cO(d^{1.5}\sqrt{\sum_{t=1}^T}\sigma_t^2+d^{1.5})$, which is suboptimal in $d$, similar to the regret bound of $\tilde\cO(d^{1.5}\sqrt{T})$ of standard TS algorithms for linear bandits \citep{agrawal2013thompson, abeille2017linear}.

\section{Preliminaries}

\subsection{Contextual Bandits}

We study the setting of contextual bandits \citep{langford2007epoch} which is an extension to multi-armed bandits by adding the notion of the context set $\cX$ and allowing the action set to be a context-dependent subset of the whole action set $\cA$.
In step $t$ over a total of $T$ steps, the agent first receives the context $x_t$ and the action set $\cA_t\subset\cA$.
The agent then selects an action $a_t\in\cA_t$ and receives randomized reward $r_t=f_*(x_t, a_t)+\epsilon_t$, where $f_*:\cX\times\cA_t\to[0, 1]$ is the ground truth reward function which is unknown to the agent, and $\epsilon_t$ is the zero-mean noise.
We use $\cF_t$ to denote the filtration generated by $\{(x_t, a_t, f_t, r_t)\}_{t\in[T]}$, i.e., all the randomness up to step $t$ (including the randomness if the reward function $f_t$ is sampled), and $\cG_t$ to denote the filtration that includes all the randomness of $\cF_t$ but $r_t$. The goal of the agent is to minimize the total regret defined as
\begin{align*}
\Regret(T)=\sum_{t=1}^T[f_*(x_t, a_t^*)-f_*(x_t, a_t^*)],
\end{align*}
where $a_t^*=\argmax_{a\in\cA_t}f_*(x_t, a)$ is the optimal action in step $t$.

\noindent\textbf{Subgaussian noise.}
Following previous works on FGTS for contextual bandits \citep{zhang2022feel,fan2023power}, we assume that the noise is subgaussian, and denote the subgaussian norm of $\epsilon_t$ as $\sigma_t^2$ which is formalized by the following assumption:
\begin{assumption}\label{assumption:subgaussian_noise}
The noise $\epsilon_t$ is $\sigma_t^2$-subgaussian conditioned on the history, i.e., for any $\lambda$, the moment-generating function of $\epsilon_t$ conditioned on $\cG_t$ satisfies
\begin{align*}
\log\EE[\exp(\lambda\epsilon_t)|\cG_t]\le\sigma_t^2\lambda^2/8.
\end{align*}
\end{assumption}
We denote $\Lambda\coloneqq\sum_{t=1}^T\sigma_t^2$. We assume that $\sigma_t^2$ is revealed to the agent at the beginning of step $t$, which is referred to as the setting of ``weak adversary with variance revealing'' in \citet{jia2024does}. Compared with the assumption used in UCB-based variance-aware algorithms \citep{zhou2021nearly, kim2022improved, zhou2022computationally, zhao2023variance, xu2023noise} where $\epsilon_t$ is bounded with variance $\EE[\epsilon_t^2|\cG_t]$, Assumption \ref{assumption:subgaussian_noise} subsumes that $\EE[\epsilon_t^2|\cG_t]=\cO(\sigma_t^2)$ (but not the boundedness of $\epsilon_t$) because
\begin{align*}
\EE[\epsilon_t^2|\cG_t]=\lim_{\lambda\to0}\frac{\EE[\exp(\lambda\epsilon_t)|\cG_t]-1}{\lambda^2/2}\le\lim_{\lambda\to0}\frac{\exp(\sigma_t^2\lambda^2/8)-1}{\lambda^2/2}=\frac{\sigma_t^2}{4}.
\end{align*}

\noindent\textbf{Reward function class.}
The agent may estimate the reward function that belongs to a function class $\cF$. We focus on the realizable setting where the ground truth reward function satisfies $f_*\in\cF$. An example of the reward function class is the family of linear functions denoted by $\cF_{d}^{\mathrm{lin}}=\{f_{\btheta}: \btheta\in\Theta\}$, where $f_{\btheta}(x, a)=\langle\btheta, \bphi(x, a)\rangle\in[0, 1]$, $\Theta\subset\{\wb\in\RR^d:\|\wb\|_2\le1\}$ is the parameter space, and $\bphi:\cX\times\cA\to\RR^d$ is a feature mapping.
We also consider the function class with finite generalized Eluder dimension \citep{agarwal2023vo}, defined as follows:
\begin{definition}
Let $Z=\{z_t\}_{t\in[T]}$ be a sequence of context-action pairs, and $\bbeta=\{\beta_t\}_{t\in[T]}$ be positive numbers. The generalized Eluder dimension of the function class $\cF$ is given by $\dim_{\lambda, \epsilon, T}(\cF)\coloneqq\sup_{Z, \bbeta\le\epsilon}\dim_\lambda(\cF, Z, \bbeta)$, where
\begin{align*}
&\dim_{\lambda}(\cF, Z, \bbeta)=\sum_{t=1}^T\min\{1, \beta_t\cD^2_{\lambda, \cF}(z_t, z_{[t-1]}, \bbeta_{[t-1]})\},\\
&\cD^2_{\lambda, \cF}(z, z_{[t-1]}, \bbeta_{[t-1]})=\sup_{f_1, f_2\in\cF}\frac{(f_1(z)-f_2(z))^2}{\lambda+\sum_{s=1}^{t-1}\beta_s(f_1(z_s)-f_2(z_s))^2}.
\end{align*}
\end{definition}

\subsection{Feel-Good Thompson Sampling}

In Thompson Sampling \citep{thompson1933likelihood}, instead of deterministically estimating the reward function, an estimation of the reward function $f_t$ is sampled in step $t$ from the posterior distribution defined as
\begin{align}
p_t^{\mathrm{TS}}(f|S_{t-1})\propto p_0(f)\exp\bigg(-\sum_{s=1}^{t-1}L_s(f, x_s, a_s, r_s)\bigg),\label{eq:vanilla_TS}
\end{align}
where $p_0(f)$ is a prior distribution, and $L_s$ is the log-likelihood often set as the square loss $L_s(f, x_s, a_s, r_s)=\eta(r_s-f(x_s, a_s))^2$ with $\eta$ being a hyperparameter. An action is then selected from $\cA_t$ to maximize $f_t(x_t, \cdot)$. However, \citet{zhang2022feel} showed that the frequentist regret of vanilla Thompson sampling is suboptimal in the worst case, and proposed Feel-Good Thompson sampling (FGTS) to fill this gap. At the core of FGTS is a modification to the posterior distribution called the \textbf{feel-good exploration} term of the form $\max_{a\in\cA_t}f(x_t, a)$. Two types of feel-good exploration terms have been developed:
\begin{itemize}[leftmargin=1.5cm]
    \item[Type A.] Augment the log-likelihood of each step with the feel-good exploration term \citep{zhang2022feel, agarwal2022model}, i.e.,
\end{itemize}
    \begin{align}\label{eq:posterior_FGTS_A}
    p_t^{\mathrm{FGTS-A}}(f|S_{t-1})\propto p_0(f)\exp\bigg(\sum_{s=1}^{t-1}[-\eta(r_s-f(x_s, a_s))^2+\textcolor{red}{\lambda\max_{a\in\cA_s}f(x_s, a)}]\bigg).
    \end{align}
\begin{itemize}[leftmargin=1.5cm]
    \item[Type B.] Only add the feel-good exploration term of the current step \citep{dann2021provably}, i.e.,
\end{itemize}
    \begin{align}\label{eq:posterior_FGTS_B}
    p_t^{\mathrm{FGTS-B}}(f|S_{t-1})\propto p_0(f)\exp\bigg(-\sum_{s=1}^{t-1}\eta(r_s-f(x_s, a_s))^2+\textcolor{red}{\lambda\max_{a\in\cA_t}f(x_t, a)}\bigg).
    \end{align}
Different from \citep{zhang2022feel} that applied Type A of FGTS to contextual bandits, our algorithm is more closely related to Type B. We will present details of our posterior distribution in Section \ref{section:algorithm}, show the effectiveness of Type B applied to (variance-aware) contextual bandits in Section \ref{section:main_results} (although it is originally developed for model-free RL), and explain the reason why Type B is preferred for variance-aware FGTS in Section \ref{section:proof_sketch}.

\section{Variance Aware Feel-Good Thompson Sampling}\label{section:algorithm}

In this section, we sketch our algorithm \FGTSVA~in Algorithm \ref{alg:FGTSVA} for contextual bandits with heterogeneous noise levels. \FGTSVA~adopts the general framework of FGTS algorithms where an estimation of the reward function $f_t$ is sampled from the posterior distribution augmented with the feel-good exploration term, and then the action $a_t$ is selected to maximize the reward function.

\begin{algorithm}[ht]
\caption{\FGTSVA}\label{alg:FGTSVA}
\begin{algorithmic}[1]
\STATE Given hyperparameter $\alpha$ and $\gamma$. Initialize $S_0=\varnothing$.
\FOR{$t=1$ \TO $T$}
\STATE Receive context $x_t$.
\STATE Set parameters $\{\eta_s\}_{s\in[t-1]}$ and $\lambda_t$ according to \eqref{eq:parameters}.
\STATE Sample $f_t\sim p_t(\cdot|S_{t-1})$, with the posterior distribution $p_t(f|S_{t-1})$ defined in \eqref{eq:posterior}.
\STATE Select $a_t=\argmax_{a\in\cA_t}f_t(x_t, a)$.
\STATE Observe reward $r_t$; update $S_t=S_{t-1}\cup\{(x_t, a_t, r_t)\}$.
\ENDFOR
\end{algorithmic}
\end{algorithm}

\noindent\textbf{Posterior distribution.}
Motivated by \eqref{eq:posterior_FGTS_B}, the posterior distribution is designed as
\begin{align}\label{eq:posterior}
p_t(f|S_{t-1})\propto p_0(f)\exp\bigg(-\sum_{s=1}^{t-1}\textcolor{red}{\eta_s}(r_s-f(x_s, a_s))^2+\textcolor{red}{\lambda_t\max_{a\in\cA_t}f(x_t, a)}\bigg),
\end{align}
where parameters $\eta_s$ and $\lambda_t$ are chosen as
\begin{align}\label{eq:parameters}
\eta_s=\bar\sigma_s^{-2},\quad\lambda_t=c\sqrt{\Lambda_t}/\bar\sigma_t^2,\quad\text{where}~\bar\sigma_t=\max\{\sigma_t, \alpha\},\quad\Lambda_t=\sum_{s=1}^t\bar\sigma_s^2.
\end{align}
with $\alpha$ and $c$ being hyperparameters. We explain the design of $\eta_t$ and $\lambda_t$ as follows:
\begin{itemize}[leftmargin=*]
\item The constructions of $\eta_s$ and $\bar\sigma_s$ are similar to \citet{zhou2021nearly}, with $\alpha>0$ being a hyperparameter that controls $\eta_s$ in case of vanishing $\sigma_s^2$ and can be set as $\cO(1/\mathrm{poly}(T))$. The coefficient $\eta_s$ functions as a preconditioner that balances the squared error $(r_s-f(x_s, a_s))^2$ across different steps. Compared with our algorithm, \citet{zhou2022computationally} and \citet{xu2023noise} set $\bar\sigma_t$ as the maximum of not only $\sigma_t$ and $\alpha$, but also a quantity called the ``uncertainty'' that depends on the action $a_t$. However, this approach is prohibitive in our algorithm due to the occurrence of $\bar\sigma_t$ in $\lambda_t$ which is required before the choice of $a_t$.
\item The parameter $\lambda_t$ controls the magnitude of the feel-good exploration term. $\lambda_t$ scales with $\bar\sigma_t^{-2}$ because intuitively, if the reward of the current step is small, then exploration should be encouraged due to the more informative reward feedback. Setting $\lambda_t=c\sqrt{\Lambda_t}/\bar\sigma_t^2$ achieves the same regret as setting $\lambda_t=c\sqrt{\Lambda}/\bar\sigma_t^2$ (explained in Section \ref{section:proof_sketch}), and avoids requiring the total variance $\Lambda$ at initialization.
\end{itemize}

\section{Main Results}\label{section:main_results}

We first introduce the generalized decoupling coefficient, an extension of the standard decoupling coefficient in \citet{dann2021provably} which is a crucial tool in the analysis of algorithms based on FGTS:
\begin{definition}[Generalized decoupling coefficient]
Let $Z=\{z_t\}_{t\in[T]}$ be a sequence of context-action pairs, and $\bbeta=\{\beta_t\}_{t\in[T]}$ be positive numbers. The decoupling coefficient is defined as $\dc_{\lambda, \epsilon, T}(\cF)\coloneqq\sup_{Z, \bbeta\le\epsilon}\dc_\lambda(\cF, Z, \bbeta)$, where $\dc_\lambda(\cF, Z, \bbeta)$ is the smallest number that satisfies
\begin{align}\label{eq:def_gen_DC}
\sum_{t=1}^T(f_t(z_t)-f_*(z_t))\le\sum_{t=1}^T\frac{\gamma}{\beta_t}\sum_{s=1}^{t-1}\beta_s(f_t(z_s)-f_*(z_s))^2+\textcolor{blue}{\gamma\lambda\sum_{t=1}^T\frac1{\beta_t}}+\textcolor{blue}{\bigg(1+\frac1{4\gamma}\bigg)}\dc_\lambda(\cF, Z, \bbeta),
\end{align}
for any sequence $\{f_t\}_{t\in[T]}$ where $f_t\in\cF$ and any $\gamma>0$.
\end{definition}
The generalized decoupling coefficient is used to analyze $f_t(x_t, a_t)-f_*(x_t, a_t)$ where $a_t$ and $f_t$ are dependent random variables because $a_t$ is selected to maximize $f_t(x_t, \cdot)$ in the algorithm design. It relates the error of the current step to the error of historic steps $f_t(z_s)-f_*(z_s)$, and generalizes the standard decoupling coefficient \citep{dann2021provably} by introducing undermined parameters $\beta_t$. We discuss the relationships of the generalized decoupling coefficient with the standard decoupling coefficient and with other complexity measures as follows:

\noindent\textbf{Relationship with standard decoupling coefficient.}
By fixing $\beta_t=1$, the generalized decoupling coefficient is closely related to the standard decoupling coefficient in \citet{dann2021provably}: For any $\gamma\le1$, the standard decoupling coefficient is defined as $\dc'_T(\cF)=\sup_{Z}\dc'(\cF, Z)$, where $\dc'(\cF, Z)$ is the smallest number that satisfies
\begin{align}\label{eq:def_DC}
\sum_{t=1}^T(f_t(z_t)-f_*(z_t))\le\gamma\sum_{t=1}^T\sum_{s=1}^T(f_t(z_s)-f_*(z_s))^2+\frac{\dc'(\cF, Z)}{4\gamma}.
\end{align}
Comparing \eqref{eq:def_gen_DC} and \eqref{eq:def_DC}, two additional terms occur in the definition of the generalized decoupling coefficient: (i) The term $\gamma\lambda\sum_{t=1}^T\beta_t^{-1}$ is a technical artifact and can be shaved by choosing small $\lambda$; (ii) The term $\dc_\lambda(\cF, Z, \bbeta)$ is shadowed by $\dc_{\lambda}(\cF, Z, \bbeta)/(4\gamma)$ when $\gamma\le1$ as is the case in \eqref{eq:def_DC}. However, as more flexible choices of both $\bbeta$ and $\gamma$ (possibly $>1$) are required, this additional term is unavoidable in the variance-aware setting. The occurrence of this term is due to the possibly large values of $\beta_t$, and will be explained in detail in Appendix \ref{section:generalized_dc} (see \eqref{eq:regret_ind_decomp}).

\noindent\textbf{Relationship with other complexity measures.}
For the linear reward function class, we can upper bound the generalized reward function as follows:
\begin{proposition}\label{prop:linear_GDC}
For the linear function class $\cF_d^{\mathrm{lin}}$, the generalized decoupling coefficient satisfies
\begin{align*}
\dc_{\lambda, \epsilon, T}(\cF_d^{\mathrm{lin}})\le2d\log(1+(\epsilon T)/(d\lambda)).
\end{align*}
\end{proposition}

For the general reward function class, the generalized decoupling coefficient can be bounded by the generalized Eluder dimension \citep{agarwal2023vo} as follows:
\begin{proposition}\label{prop:general_GDC}
For a reward function class with finite generalized Eluder dimension, the generalized decoupling coefficient satisfies $\dc_{\lambda, \epsilon, T}(\cF)\le\dim_{\lambda, \epsilon, T}(\cF)$.
\end{proposition}

The proofs of Propositions \ref{prop:linear_GDC} and \ref{prop:general_GDC} are given in Appendix \ref{section:generalized_dc}.

We now present the regret upper bound of \FGTSVA:
\begin{theorem}\label{theorem:regret_upper_bound}
Suppose that Assumption \ref{assumption:subgaussian_noise} holds, the function class $\cF$ has finite cardinality, and the prior distribution $p_0$ is the uniform distribution on $\cF$. Assume that parameters $\eta_t$ and $\lambda_t$ are chosen according to \eqref{eq:parameters}, and the hyperparameters are chosen as
\begin{align*}
\alpha=1/\sqrt{T}, \quad\lambda=1,\quad\epsilon=\alpha^{-2},\quad c=2\sqrt{\dc_{\lambda, \epsilon, T}^{-1}(\cF)\log|\cF|}.
\end{align*}
Then the total regret of \FGTSVA~satisfies
\begin{align*}
\EE[\Regret(T)]\lesssim\sqrt{(1+\Lambda)\dc_{\lambda, \epsilon, T}(\cF)\log|\cF|}+\dc_{\lambda, \epsilon, T}(\cF).
\end{align*}
\end{theorem}
The proof of Theorem \ref{theorem:regret_upper_bound}, as well as a more general version of Theorem \ref{theorem:regret_upper_bound} that extends to the case of infinite function class, is given in Appendix \ref{section:proof_variance_upper_bound}.
\begin{remark}
Under the setting of linear contextual bandits, $\cF$ can be chosen as the $\varepsilon$-net of the unit ball whose cardinality satisfies $\log|\cF|=\tilde\cO(d)$. Additionally, the generalized decoupling coefficient satisfies $\dc=\tilde\cO(d)$ as is shown in Proposition \ref{prop:linear_GDC}. Therefore, when applied to linear contextual bandits, \FGTSVA~has a \textbf{nearly-optimal regret} of $\tilde\cO(d\sqrt\Lambda+d)$, similar to UCB-based algorithms \citep{zhou2022computationally, zhao2023variance}.
\end{remark}
\begin{remark}
Under the setting of deterministic reward, the total variance is $\Lambda=0$, and the regret of \FGTSVA~is $\cO(\dc)$\footnote{Although there is a $1+\Lambda$ term in the regret bound of Theorem \ref{theorem:regret_upper_bound}, it can be further suppressed by setting $\alpha$ to be an even smaller number.}. Note that the generalized decoupling coefficient is upper bounded by the generalized Eluder dimension, which reduces to the standard Eluder dimension \citep{russo2013eluder} in the deterministic case. Therefore, the regret of \FGTSVA~is \textbf{minimax-optimal} in the deterministic case for the general reward function class \citep{jia2024does}.
\end{remark}
\begin{remark}
When $\sigma_t^2=1$ for all $t\in[T]$, the setting reduces to the standard contextual bandits, and the reward of \FGTSVA~is $\tilde\cO(\sqrt{T\dc\cdot\log|\cF|}+\dc)$, which is $\tilde\cO(d\sqrt T)$ for linear contextual bandits.
Therefore, although \citet{dann2021provably} only studied FGTS in model-free RL, we have shown that a similar posterior distribution in \eqref{eq:posterior_FGTS_B}, is applicable to the setting of contextual bandits, and enjoys the minimax-optimal regret bound for the linear regret function class similar to \citet{zhang2022feel}.
Additionally, since the parameters are $\eta_t=1$ and $\lambda_t=\Theta(\dc^{-1}\sqrt t\log|\cF|)$, \FGTSVA~is reduced to the \textbf{first FGTS-based algorithm that does not require knowledge of the horizon} $T$.
\end{remark}

\section{Overview of Proof}\label{section:proof_sketch}

We first define several shorthand notations: We use $\dc$ to denote the (generalized) decoupling coefficient, and denote
\begin{gather*}
\Delta L(f, x, a, r)=(r-f(x, a))^2-(r-f_*(x, a))^2,\quad\FG_t(f)=\max_{a\in\cA_t}f(x_t, a)-f_*(x_t, a_t^*),\\
\LS_t(f)=(f(x_t, a_t)-f_*(x_t, a_t))^2,
\end{gather*}
then the posterior distribution \eqref{eq:posterior} is equivalent to
\begin{align}\label{eq:posterior_equiv}
p(f|S_{t-1})\propto p_0(f)\exp\bigg(-\sum_{s=1}^{t-1}\Delta L(f, x_s, a_s, r_s)+\lambda_t\FG_t(f)\bigg).
\end{align}
Note that the regret at step $t$ can be decomposed as
\begin{align*}
\EE[f_*(x_t, a_t^*)-f_*(x_t, a_t)]&=\EE[f_t(x_t, a_t)-f_*(x_t, a_t)]-[f_t(x_t, a_t)-f_*(x_t, a_t^*)]\\
&=\underbrace{\EE[f_t(x_t, a_t)-f_*(x_t, a_t)]}_{\text{Bellman Error}}-\EE[\FG_t(f_t)],
\end{align*}
where the second equality holds because $a_t$ is the maximizer of $f_t(x_t, \cdot)$. The Bellman error term $\EE[f_t(x_t, a_t)-f_*(x_t, a_t)]$ is then bounded using the (generalized) decoupling coefficient, which has two versions corresponding to the two types of posterior distributions. In the remaining of this section, we will explain the reason why Type B of the posterior distribution in \eqref{eq:posterior_FGTS_B} is the basis of our algorithm instead of Type A in \eqref{eq:posterior_FGTS_A} by first explaining the obstacles in the analysis based on Type A, and then showing how \FGTSVA~built on Type B manages to overcome the obstacle. We will also explain the technical trick in the construction of $\lambda_t$.

\subsection{Technical Obstacle when Applying Type A of FGTS}

When applying the posterior distribution similar to \eqref{eq:posterior_FGTS_A}, i.e.
\begin{align*}
p_t(f|S_{t-1})\propto p_0(f)\exp\bigg(\sum_{s=1}^{t-1}[-\eta_s\Delta L(f, x_s, a_s, r_s)+\lambda_s\FG_s(f)]\bigg),
\end{align*}
the proof developed based on \citet{zhang2022feel} requires that $\eta_t$ is upper bounded by an absolute constant that is irrelevant to $T$, which results in the regret bound polynomial in $T$ because of the occurrence of $\sum1/\eta_t$ in the regret bound. In detail, by applying the decoupling coefficient for this type of posterior distribution, we have
\begin{align*}
&\EE[f(x_t, a_t)-f_*(x_t, a_t)]\le\frac{\dc}{4\gamma_t}+\gamma_t\EE_{S_{t-1}, x_t}\EE_{a_t|S_{t-1}, x_t}\EE_{\tilde f\sim p_t}\LS_t(\tilde f).
\end{align*}
Therefore, it suffices to prove an upper bound for
\begin{align}\label{eq:desired_upper bound}
\sum_{t=1}^T\frac{\dc}{4\gamma_t}+\sum_{t=1}^T\EE_{S_{t-1}, x_t}\EE_{\tilde f\sim p_t}\big[\gamma_t\EE_{a_t|S_{t-1}, x_t}\LS_t(\tilde f)-\FG_t(\tilde f)\big].
\end{align}
We define the potential as
\begin{align*}
Z_t=\EE_{S_t}\log\EE_{f\sim p_0}\exp\bigg(\sum_{s=1}^{t}[-\eta_s\Delta L(f, x_s, a_s, r_s)+\lambda_s\FG_s(f, x_s)]\bigg).
\end{align*}
The proof proceeds by bounding $Z_t-Z_{t-1}$ and applying the telescope sum. Note that
\begin{align*}
&Z_t-Z_{t-1}=\EE_{S_t}\log\EE_{\tilde f\sim p_t}\exp\big(-\eta_s\Delta L(f, x_t, a_t, r_t)+\lambda_t\FG_t(f, x_s)\big)\\
&\le\frac12\EE_{S_t}\Big[\underbrace{\log\EE_{\tilde f\sim p(\cdot|S_{t-1})}\exp(-2\eta_s\Delta L(f, x_t, a_t, r_t))}_{I_1}+\underbrace{\log\EE_{\tilde f\sim p(\cdot|S_{t-1)}}\exp(2\lambda_t\FG_t(\tilde f, x_t))}_{I_2}\Big],
\end{align*}
where the inequality holds due to H\"older's inequality. By using the Hoeffding's Lemma, the term $I_2$ can be bounded by $2\lambda_t\EE_{\tilde f\sim p(\cdot|S_{t-1})}\FG_t(\tilde f, x_t)+2\lambda_t^2$. For the term $I_1$, By first taking the conditional expectation on $\cG_t$ and using Assumption \ref{assumption:subgaussian_noise}, we have
\begin{align}\label{eq:EE_I1|Gt}
\EE[I_1|\cG_t]&\le\log\EE_{\tilde f\sim p_t}\exp(-2\eta_t(1-\sigma_t^2\eta_t)\LS_t(\tilde f)).
\end{align}
By choosing $\eta_t\le\sigma_t^{-2}/2$, the coefficient $2\eta_t(1-\sigma_t^2\eta_t)$ can be lower bounded by $\eta_t$. To connect \eqref{eq:EE_I1|Gt} with \eqref{eq:desired_upper bound}, the RHS of \eqref{eq:EE_I1|Gt} has to be bounded by $-C\eta_t\cdot\EE_{\tilde f\sim p(\cdot|S_{t-1})}\LS_t(\tilde f)$ where $C$ is an absolute constant. This is possible only when $\eta_t\LS_t(\tilde f)=\cO(1)$. Since $\LS_t(\tilde f)=\Theta(1)$ in the worst case, $\eta_t$ should be an absolute constant. Comparing what we obtain against \eqref{eq:desired_upper bound}, we note that (i) $\lambda_t$ has to be a constant $\lambda$ to enable the telescope sum of $Z_t-Z_{t-1}$, and (ii) to make coefficients match, we require $\gamma_t=C\eta_t/\lambda=\Theta(\lambda^{-1})$. Thus, the first term of \eqref{eq:desired_upper bound} becomes $\cO(\lambda T\dc)$, with an undesirable factor of $T$. Therefore, Type A of FGTS in \eqref{eq:posterior_FGTS_A} cannot yield variance-aware regret bounds with existing techniques, even if inhomogeneous parameters $\eta_t$ and $\lambda_t$ are allowed.

\subsection{Highlight of Proof Techniques}
By proceeding through a sharply different path, the analysis of \FGTSVA~avoids the aforementioned obstacle that stems from bounding the expectation of the exponential term on RHS of \eqref{eq:EE_I1|Gt}. The following two techniques work together to relate the posterior distribution with the desired form of the decoupling coefficient:

\noindent\textbf{Technique 1: Prioritizing expectation over the randomness of reward.}
Although the expectation of an exponential term over the randomness of posterior sampling causes trouble, the expectation over the randomness of $\epsilon_t$ adopts a simple form due to Assumption \ref{assumption:subgaussian_noise}. Therefore, by defining
\begin{align*}
\xi_s(\tilde f, x_s, a_s, r_s)=-\eta_s\Delta L(\tilde f, x_s, a_s, r_s)-\log\EE[\exp(-\eta_s\Delta L(\tilde f, x_s, a_s, r_s))|\cG_s],
\end{align*}
we have the following property (following \citet{dann2021provably}; formalized in Lemma \ref{lemma:Exp_exp=1}):
\begin{align*}
\EE_{S_t}\exp\bigg(\sum_{s=1}^t\xi_s(\tilde f, x_s, a_s, r_s)\bigg)=1.
\end{align*}
By using the Jensen's inequality, we have
\begin{align}
0&=\log\EE_{\tilde f\sim p_0}\EE_{S_{t-1}, x_t}\exp\bigg(\sum_{s=1}^{t-1}\xi_s(\tilde f, x_s, a_s, r_s)\bigg)\nonumber\\
&\ge\EE_{S_{t-1}, x_t}\log\EE_{\tilde f\sim p_0}\exp\bigg(\sum_{s=1}^{t-1}\xi_s(\tilde f, x_s, a_s, r_s)\bigg)\label{eq:ind_Exp_exp=1}.
\end{align}

\noindent\textbf{Technique 2: KL-regularized optimality.}
The following lemma is used to remove the exponential on the RHS of \eqref{eq:ind_Exp_exp=1}:
\begin{lemma}[Donsker–Varadhan duality, see e.g., Proposition 7.16 in \citet{zhang2023mathematical}]\label{lemma:KL_reg}
Let $(\cX, \cF, P_0)$ be a probability space and $U(x)$ be a measurable function. Then for any distribution $P$ on $(\cX, \cF)$, we have
\begin{align*}
\EE_{x\sim P}[U(x)]+\KL(P||P_0)\ge-\log\EE_{x\sim P_0}\exp(-U(x)),
\end{align*}
and the infimum is attained when $P(x)\propto P_0(x)\exp(-U(x))$.
\end{lemma}
The RHS of Lemma \ref{lemma:KL_reg} contains the expectation of the exponential term, similar to the RHS of \eqref{eq:ind_Exp_exp=1}, and the LHS is the simple expectation of $U(\cdot)$, free of exponential terms.
The price of the removal of the exponential is an additional KL-divergence term, so we use Lemma \ref{lemma:KL_reg} twice to cancel out the KL-divergence, one using the inequality itself, and the other using the optimality condition:
\begin{gather}
\log\EE_{\tilde f\sim p_0}\exp\bigg(\sum_{s=1}^{t-1}\xi_s(\tilde f, x_s, a_s, r_s)\bigg)\ge\sum_{s=1}^{t-1}\EE_{\tilde f\sim p_t}\xi_s(\tilde f, x_s, a_s, r_s)-\KL(p_t||p_0);\label{eq:KL_reg_bound}\\
-\KL(\delta_{f_*}||p_0)\le\EE_{\tilde f\sim p_t}\bigg[-\sum_{s=1}^{t-1}\eta_s\Delta L(\tilde f, x_s, a_s, r_s)+\lambda_t\FG_t(\tilde f)\bigg]-\KL(p_t||p_0).\label{eq:KL_reg_optimal}
\end{gather}
Plugging \eqref{eq:KL_reg_bound} and \eqref{eq:KL_reg_optimal} into \eqref{eq:ind_Exp_exp=1}, noting that $\KL(\delta_{f_*}||p_0)=\log|\cF|$, we have
\begin{align}
\log|\cF|&\ge\EE_{S_{t-1}, x_t}\EE_{\tilde f\sim p_t}\bigg[-\sum_{s=1}^{t-1}\log\EE[\exp(-\eta_s\Delta L(\tilde f, x_s, a_s, r_s)|\cG_s]-\lambda_t\FG_t(\tilde f)\bigg]\nonumber\\
&\ge\EE_{S_{t-1}, x_t}\EE_{\tilde f\sim p_t}\bigg[\sum_{s=1}^{t-1}\frac{\eta_s}{2}\LS_s(\tilde f)-\lambda_t\FG_t(\tilde f)\bigg],\label{eq:proof_sketch_main_ineq}
\end{align}
where the second inequality holds due to Assumption \ref{assumption:subgaussian_noise} with an argument similar to \eqref{eq:EE_I1|Gt}. \eqref{eq:proof_sketch_main_ineq} is thus completely free of the expectation of exponential terms.

\noindent\textbf{Avoiding $\Lambda$ in parameters.}
If we follow the proof of \citet{dann2021provably}, then we require $\lambda_t=1/(2\gamma\bar\sigma_t^2)$ where $\gamma$ scales with $\Lambda^{-0.5}$, which is unknown to the agent.
To resolve this issue, we observe that the proof can proceed by replacing the total variance $\Lambda$ with the partial sum $\Lambda_t$.
Specifically, dividing $\lambda_t=c\sqrt{\Lambda_t}/\bar\sigma_t^2$ on both sides of the \eqref{eq:proof_sketch_main_ineq}, noting that $\Lambda_t\le\Lambda_T$, we have
\begin{align*}
\frac{\bar\sigma_t^2}{c\sqrt{\Lambda_t}}\log|\cF|\ge\EE_{S_{t-1}, x_t}\EE_{\tilde f\sim p_t}\bigg[\frac{\bar\sigma_t^2}{2c\sqrt{\Lambda_T}}\sum_{s=1}^{t-1}\bar\sigma_s^{-2}\LS_s(\tilde f)-\FG_t(\tilde f)\bigg].
\end{align*}
Plugging the inequality into the definition of the decoupling coefficient, we have
\begin{align*}
\EE[\Regret(T)]\lesssim\frac{\log|\cF|}{c}\sum_{t=1}^T\frac{\bar\sigma_t^2}{\sqrt{\Lambda_t}}+\bigg(1+\frac{c\sqrt{\Lambda_T}}{2}\bigg)\dc\le\frac{2\sqrt{\Lambda_T}\log|\cF|}{c}+\bigg(1+\frac{c\sqrt{\Lambda_T}}{2}\bigg)\dc,
\end{align*}
where we use the crucial technical lemma (Lemma \ref{lemma:sum_1/gamma_t}) stating that $\sum_{t=1}^T\bar\sigma_t^2/\sqrt{\Lambda_t}\le2\sqrt{\Lambda_T}$. Finally, by choosing $c=2\sqrt{\dc^{-1}\log|\cF|}$ (which is irrelevant to $\Lambda$) and noting that $\Lambda_T=\Theta(\alpha^2T+\Lambda)$, the regret can be bounded by $\cO(\sqrt{(\alpha^2T+\Lambda)\dc\cdot\log|\cF|}+\dc)$.

\section{Experiments}

In this section, we examine our algorithm, \FGTSVA, against baselines (including Weighted OFUL+, \texttt{FGTS}, and SAVE) in experiments with synthetic data. The code can be found at \url{https://github.com/xuheng-li99/FGTS-VA}.

\noindent\textbf{Environment.}
We focus on the setting of linear bandits with $d=5$ and $\cX=\{x\}$, so we omit the context $x$ for simplicity. The action set is $\cA_t=\cA=\{\pm1/\sqrt d\}^{d}$, and the ground truth parameter $\btheta_*$ is sampled from the uniform distribution on the unit sphere. We consider two noise models with heterogeneous noise magnitudes. In both cases, the noise $\epsilon_t$ is sampled from $\cN(0, \sigma_t^2)$.
\begin{itemize}[leftmargin=*]
\item[1.] The noise is sparse: $\sigma_t^2=1$ with probability $p$, and $\sigma_t^2=0$ with probability $1-p$. We set $p=0.1$ in our experiments.
\item[2.] The noise is dense: $\sigma_t^2$ is sampled from a $\chi^2$ distribution with degree of freedom equal to 1.
\end{itemize}

\noindent\textbf{Implementation details.} For \FGTSVA, in the linear bandit setting, we let the prior distribution be the Gaussian distribution $\cN(\zero, \Ib_d/d)$.
We use Langevin dynamics to sample from this distribution:
\begin{align*}
\btheta_t^{(k+1)}=\btheta_t^{(k)}+\delta^{(k)}\nabla\log p(\btheta|S_{t-1})+\sqrt{2\delta^{(k)}}\bepsilon_t,
\end{align*}
where $\bepsilon_t$ is the standard Gaussian noise, and $\delta^{(k)}$ is the stepsize. We use $K=20$ SGLD steps in our experiments, and initialize $\btheta_{t+1}^{(0)}=\btheta_{t}^{(K)}$.

\begin{figure}[ht]
\centering
\subfigure[Sparse noise.]{\includegraphics*[width=0.4\linewidth]{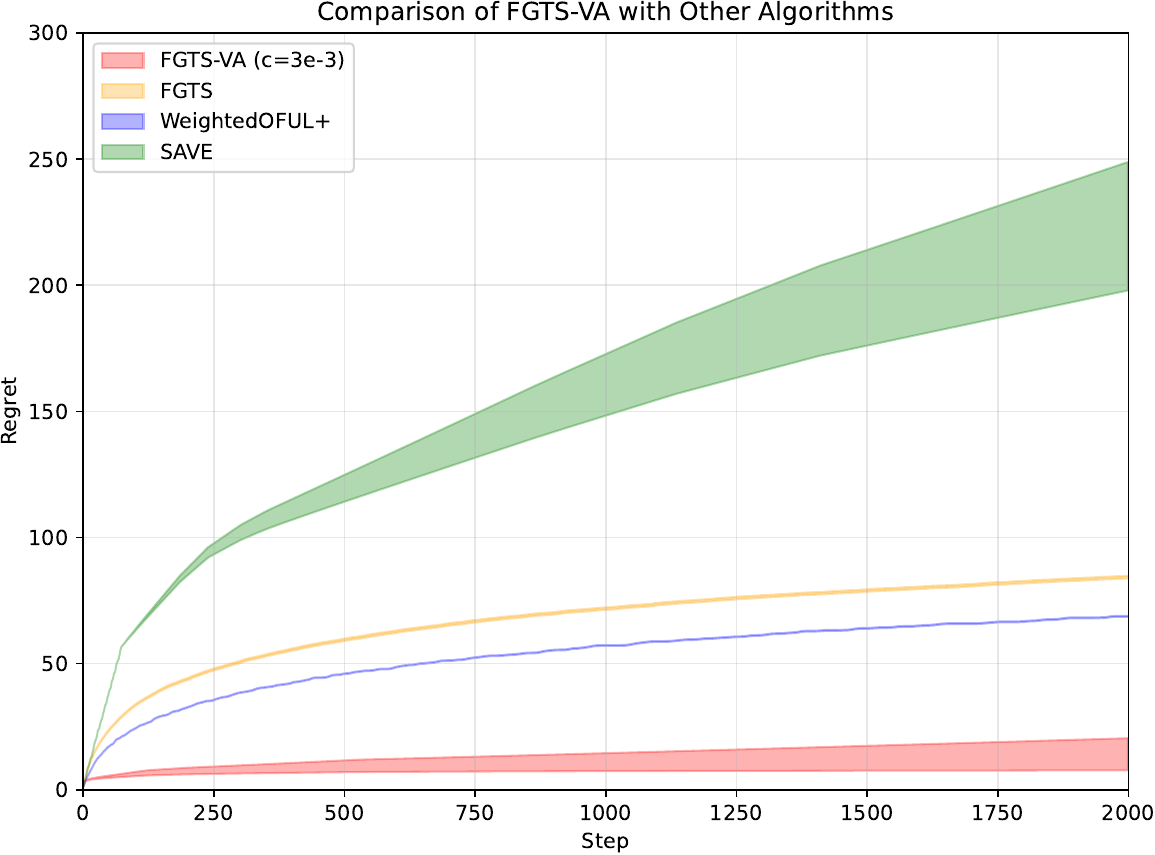}}
\subfigure[Dense noise.]{\includegraphics*[width=0.4\linewidth]{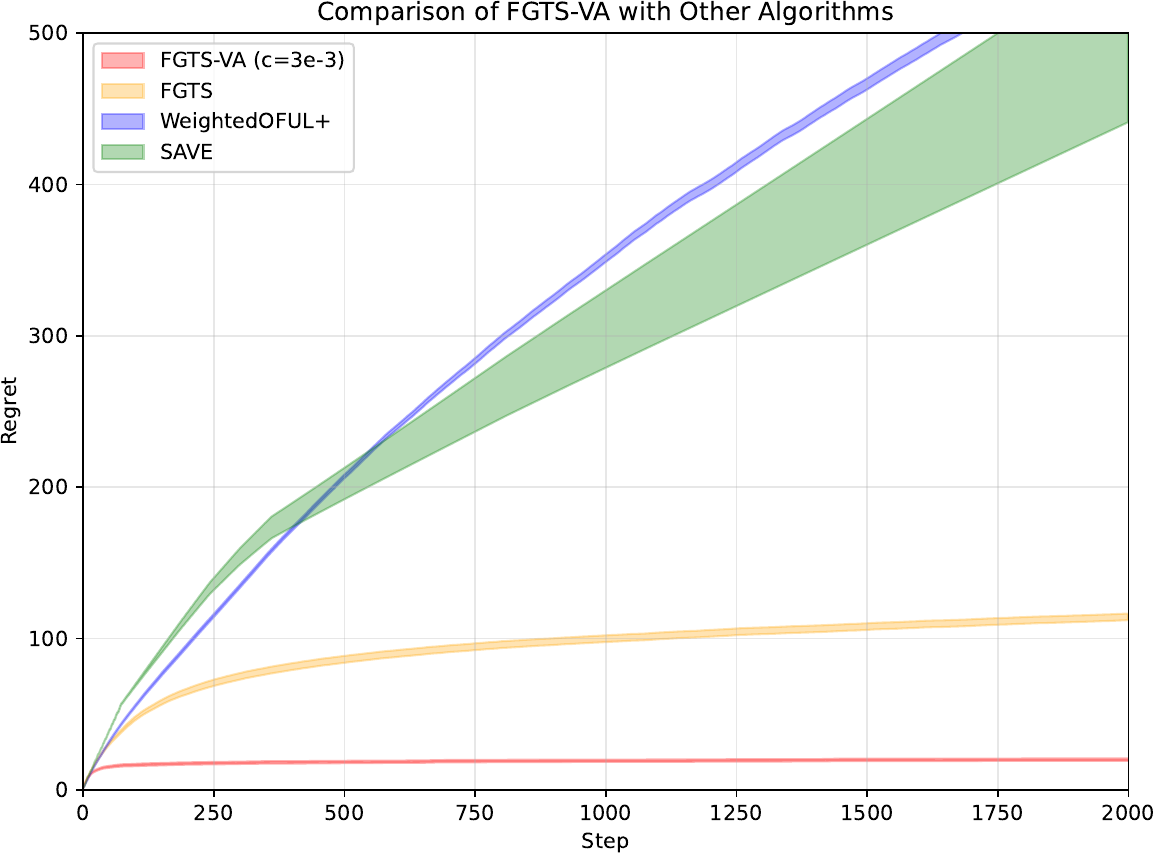}}
\caption{Comparison of different algorithms. Error bands are plotted over 100 runs.}
\label{fig:compare_alg}
\end{figure}
\noindent\textbf{Comparison of different algorithms.}
We first compare \FGTSVA~with $c=0.003$ against Weighted OFUL+ \citep{zhou2022computationally}, \texttt{SAVE} \citep{zhao2023variance}, and \texttt{FGTS} \citep{zhang2022feel} with results in Figure \ref{fig:compare_alg}. For both data models, \FGTSVA~outperforms the baselines by a large margin.

\begin{figure}[ht]
\centering
\subfigure[Sparse noise.]{\includegraphics*[width=0.4\linewidth]{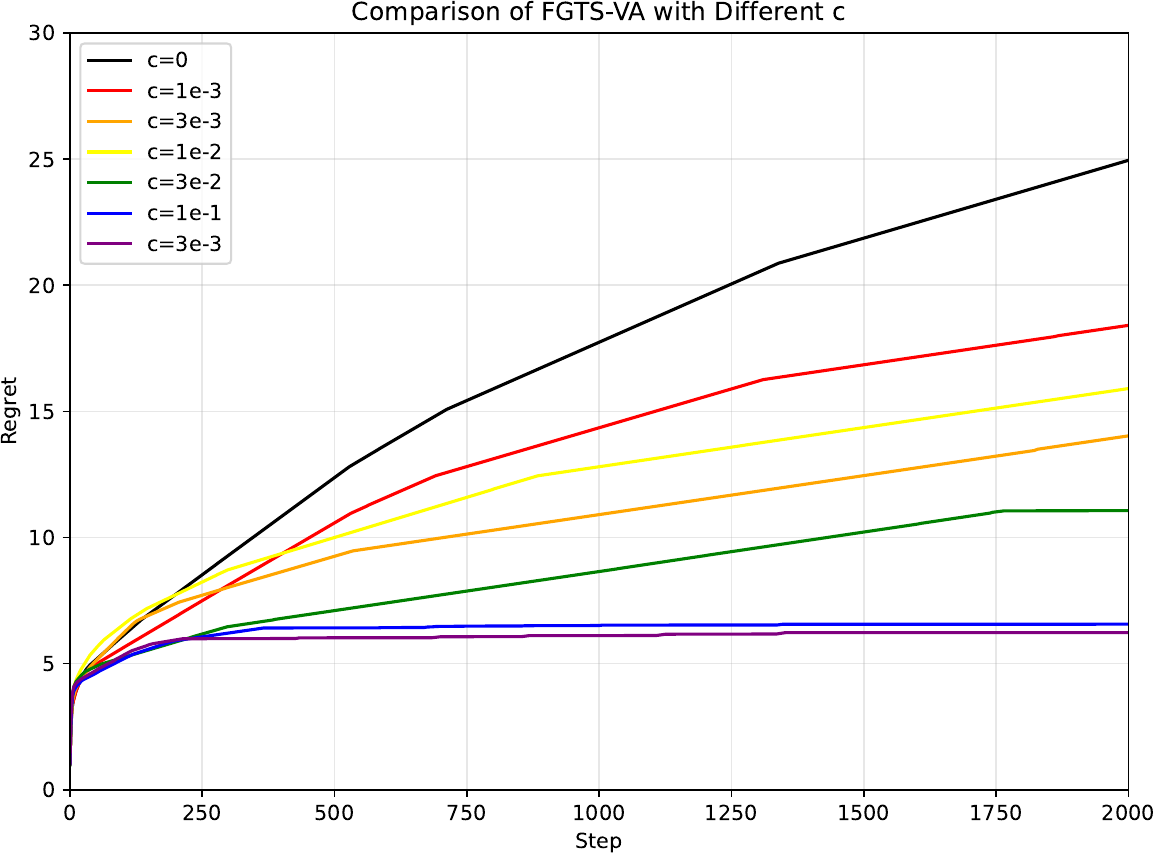}}
\subfigure[Dense noise.]{\includegraphics*[width=0.4\linewidth]{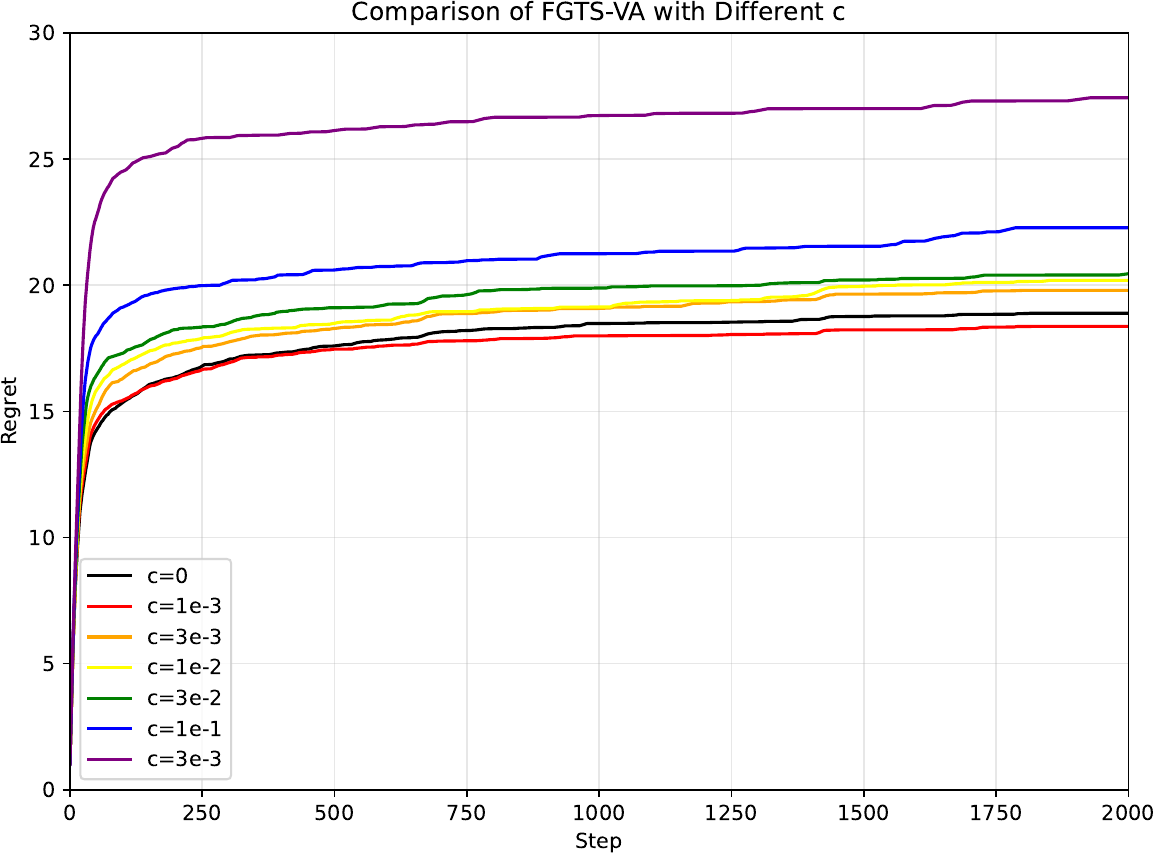}}
\caption{Comparison of different choices of $c$. The averages of regret over 100 runs are plotted.}
\label{fig:compare_c}
\end{figure}
\noindent\textbf{Ablation studies.}
We then perform ablation studies of the algorithm with different choices of $c$. It is worth noting that $c$ is the only tunable parameter of \FGTSVA, and $c=\tilde\Theta(1)$ for linear bandits according to Theorem \ref{theorem:regret_upper_bound}. The results are shown in Figure \ref{fig:compare_c}. For the case of sparse noise, we observe the advantage of choosing $c$ bounded away from $0$, i.e., advantage of the feel-good exploration. For the case of dense noise, the optimal choice of $c$ is close to $0$.

\section{Conclusion}

In this work, we present \FGTSVA, a variance-aware algorithm for general contextual bandits based on Feel-Good Thompson sampling. In the posterior distribution, we incorporate not only variance-related weights, but also a feel-good exploration term that adopts the idea from model-free RL \citep{dann2021provably}. The generalized decoupling coefficient is the pivotal technique in our analysis, with which we show that \FGTSVA~achieves a nearly optimal regret bound similar to UCB-based algorithms. A restriction of our work is the setting of variance revealing, so it is interesting to explore the possibility of designing FGTS-based algorithms without requiring the variance of the current step as in \citet{zhou2022computationally}, and ultimately without knowing the variance at all similar to \citet{zhao2023variance}. Extending of the techniques in this work to reinforcement learning is also an interesting future direction.

\section*{Acknowledgment}
We thank the anonymous reviewers and area chair for their helpful comments. XL and QG are supported in part by the National Science Foundation DMS-2323113 and IIS-2403400.
The views and conclusions contained in this paper are those of the authors and should not be interpreted as representing any funding agencies.

\bibliographystyle{ims}
\bibliography{camera_ready}

\begin{thebibliography}{41}
\expandafter\ifx\csname natexlab\endcsname\relax\def\natexlab#1{#1}\fi
\expandafter\ifx\csname url\endcsname\relax
  \def\url#1{\texttt{#1}}\fi
\expandafter\ifx\csname urlprefix\endcsname\relax\def\urlprefix{URL }\fi

\bibitem[{Abbasi-Yadkori et~al.(2011)Abbasi-Yadkori, P{\'a}l and Szepesv{\'a}ri}]{abbasi2011improved}
\textsc{Abbasi-Yadkori, Y.}, \textsc{P{\'a}l, D.} and \textsc{Szepesv{\'a}ri, C.} (2011).
\newblock Improved algorithms for linear stochastic bandits.
\newblock \textit{Advances in neural information processing systems} \textbf{24}.

\bibitem[{Abeille and Lazaric(2017)}]{abeille2017linear}
\textsc{Abeille, M.} and \textsc{Lazaric, A.} (2017).
\newblock Linear thompson sampling revisited.
\newblock In \textit{Artificial Intelligence and Statistics}. PMLR.

\bibitem[{Agarwal et~al.(2023)Agarwal, Jin and Zhang}]{agarwal2023vo}
\textsc{Agarwal, A.}, \textsc{Jin, Y.} and \textsc{Zhang, T.} (2023).
\newblock Vo $ q $ l: Towards optimal regret in model-free rl with nonlinear function approximation.
\newblock In \textit{The Thirty Sixth Annual Conference on Learning Theory}. PMLR.

\bibitem[{Agarwal and Zhang(2022)}]{agarwal2022model}
\textsc{Agarwal, A.} and \textsc{Zhang, T.} (2022).
\newblock Model-based rl with optimistic posterior sampling: Structural conditions and sample complexity.
\newblock \textit{Advances in Neural Information Processing Systems} \textbf{35} 35284--35297.

\bibitem[{Agrawal and Goyal(2012)}]{agrawal2012analysis}
\textsc{Agrawal, S.} and \textsc{Goyal, N.} (2012).
\newblock Analysis of thompson sampling for the multi-armed bandit problem.
\newblock In \textit{Conference on learning theory}. JMLR Workshop and Conference Proceedings.

\bibitem[{Agrawal and Goyal(2013)}]{agrawal2013thompson}
\textsc{Agrawal, S.} and \textsc{Goyal, N.} (2013).
\newblock Thompson sampling for contextual bandits with linear payoffs.
\newblock In \textit{International conference on machine learning}. PMLR.

\bibitem[{Agrawal and Goyal(2017)}]{agrawal2017near}
\textsc{Agrawal, S.} and \textsc{Goyal, N.} (2017).
\newblock Near-optimal regret bounds for thompson sampling.
\newblock \textit{Journal of the ACM (JACM)} \textbf{64} 1--24.

\bibitem[{Audibert et~al.(2009)Audibert, Munos and Szepesv{\'a}ri}]{audibert2009exploration}
\textsc{Audibert, J.-Y.}, \textsc{Munos, R.} and \textsc{Szepesv{\'a}ri, C.} (2009).
\newblock Exploration--exploitation tradeoff using variance estimates in multi-armed bandits.
\newblock \textit{Theoretical Computer Science} \textbf{410} 1876--1902.

\bibitem[{Auer(2002)}]{auer2002using}
\textsc{Auer, P.} (2002).
\newblock Using confidence bounds for exploitation-exploration trade-offs.
\newblock \textit{Journal of Machine Learning Research} \textbf{3} 397--422.

\bibitem[{Chapelle and Li(2011)}]{chapelle2011empirical}
\textsc{Chapelle, O.} and \textsc{Li, L.} (2011).
\newblock An empirical evaluation of thompson sampling.
\newblock \textit{Advances in neural information processing systems} \textbf{24}.

\bibitem[{Chu et~al.(2011)Chu, Li, Reyzin and Schapire}]{chu2011contextual}
\textsc{Chu, W.}, \textsc{Li, L.}, \textsc{Reyzin, L.} and \textsc{Schapire, R.} (2011).
\newblock Contextual bandits with linear payoff functions.
\newblock In \textit{Proceedings of the fourteenth international conference on artificial intelligence and statistics}. JMLR Workshop and Conference Proceedings.

\bibitem[{Dann et~al.(2021)Dann, Mohri, Zhang and Zimmert}]{dann2021provably}
\textsc{Dann, C.}, \textsc{Mohri, M.}, \textsc{Zhang, T.} and \textsc{Zimmert, J.} (2021).
\newblock A provably efficient model-free posterior sampling method for episodic reinforcement learning.
\newblock \textit{Advances in Neural Information Processing Systems} \textbf{34} 12040--12051.

\bibitem[{Di et~al.(2023)Di, Jin, Wu, Zhao, Farnoud and Gu}]{di2023variance}
\textsc{Di, Q.}, \textsc{Jin, T.}, \textsc{Wu, Y.}, \textsc{Zhao, H.}, \textsc{Farnoud, F.} and \textsc{Gu, Q.} (2023).
\newblock Variance-aware regret bounds for stochastic contextual dueling bandits.
\newblock \textit{arXiv preprint arXiv:2310.00968} .

\bibitem[{Fan and Gu(2023)}]{fan2023power}
\textsc{Fan, Z.} and \textsc{Gu, Q.} (2023).
\newblock The power of feel-good thompson sampling: A unified framework for linear bandits .

\bibitem[{Hazan and Kale(2011)}]{hazan2011better}
\textsc{Hazan, E.} and \textsc{Kale, S.} (2011).
\newblock Better algorithms for benign bandits.
\newblock \textit{Journal of Machine Learning Research} \textbf{12}.

\bibitem[{Ito(2021)}]{ito2021parameter}
\textsc{Ito, S.} (2021).
\newblock Parameter-free multi-armed bandit algorithms with hybrid data-dependent regret bounds.
\newblock In \textit{Conference on Learning Theory}. PMLR.

\bibitem[{Ito and Takemura(2023)}]{ito2023best}
\textsc{Ito, S.} and \textsc{Takemura, K.} (2023).
\newblock Best-of-three-worlds linear bandit algorithm with variance-adaptive regret bounds.
\newblock In \textit{The Thirty Sixth Annual Conference on Learning Theory}. PMLR.

\bibitem[{Jia et~al.(2024)Jia, Qian, Rakhlin and Wei}]{jia2024does}
\textsc{Jia, Z.}, \textsc{Qian, J.}, \textsc{Rakhlin, A.} and \textsc{Wei, C.-Y.} (2024).
\newblock How does variance shape the regret in contextual bandits?
\newblock \textit{Advances in Neural Information Processing Systems} \textbf{37} 83730--83785.

\bibitem[{Jin et~al.(2021)Jin, Xu, Shi, Xiao and Gu}]{jin2021mots}
\textsc{Jin, T.}, \textsc{Xu, P.}, \textsc{Shi, J.}, \textsc{Xiao, X.} and \textsc{Gu, Q.} (2021).
\newblock Mots: Minimax optimal thompson sampling.
\newblock In \textit{International Conference on Machine Learning}. PMLR.

\bibitem[{Kaufmann et~al.(2012)Kaufmann, Korda and Munos}]{kaufmann2012thompson}
\textsc{Kaufmann, E.}, \textsc{Korda, N.} and \textsc{Munos, R.} (2012).
\newblock Thompson sampling: An asymptotically optimal finite-time analysis.
\newblock In \textit{International conference on algorithmic learning theory}. Springer.

\bibitem[{Kim et~al.(2022)Kim, Yang and Jun}]{kim2022improved}
\textsc{Kim, Y.}, \textsc{Yang, I.} and \textsc{Jun, K.-S.} (2022).
\newblock Improved regret analysis for variance-adaptive linear bandits and horizon-free linear mixture mdps.
\newblock \textit{Advances in Neural Information Processing Systems} \textbf{35} 1060--1072.

\bibitem[{Langford and Zhang(2007)}]{langford2007epoch}
\textsc{Langford, J.} and \textsc{Zhang, T.} (2007).
\newblock The epoch-greedy algorithm for multi-armed bandits with side information.
\newblock \textit{Advances in neural information processing systems} \textbf{20}.

\bibitem[{Li et~al.(2010)Li, Chu, Langford and Schapire}]{li2010contextual}
\textsc{Li, L.}, \textsc{Chu, W.}, \textsc{Langford, J.} and \textsc{Schapire, R.~E.} (2010).
\newblock A contextual-bandit approach to personalized news article recommendation.
\newblock In \textit{Proceedings of the 19th international conference on World wide web}.

\bibitem[{Li et~al.(2024)Li, Zhao and Gu}]{li2024feel}
\textsc{Li, X.}, \textsc{Zhao, H.} and \textsc{Gu, Q.} (2024).
\newblock Feel-good thompson sampling for contextual dueling bandits.
\newblock \textit{arXiv preprint arXiv:2404.06013} .

\bibitem[{Liu et~al.(2023)Liu, Netrapalli, Szepesvari and Jin}]{liu2023optimistic}
\textsc{Liu, Q.}, \textsc{Netrapalli, P.}, \textsc{Szepesvari, C.} and \textsc{Jin, C.} (2023).
\newblock Optimistic mle: A generic model-based algorithm for partially observable sequential decision making.
\newblock In \textit{Proceedings of the 55th Annual ACM Symposium on Theory of Computing}.

\bibitem[{Mukherjee et~al.(2018)Mukherjee, Naveen, Sudarsanam and Ravindran}]{mukherjee2018efficient}
\textsc{Mukherjee, S.}, \textsc{Naveen, K.}, \textsc{Sudarsanam, N.} and \textsc{Ravindran, B.} (2018).
\newblock Efficient-ucbv: An almost optimal algorithm using variance estimates.
\newblock In \textit{Proceedings of the AAAI Conference on Artificial Intelligence}, vol.~32.

\bibitem[{Osband and Van~Roy(2017)}]{osband2017posterior}
\textsc{Osband, I.} and \textsc{Van~Roy, B.} (2017).
\newblock Why is posterior sampling better than optimism for reinforcement learning?
\newblock In \textit{International conference on machine learning}. PMLR.

\bibitem[{Russo and Van~Roy(2013)}]{russo2013eluder}
\textsc{Russo, D.} and \textsc{Van~Roy, B.} (2013).
\newblock Eluder dimension and the sample complexity of optimistic exploration.
\newblock \textit{Advances in Neural Information Processing Systems} \textbf{26}.

\bibitem[{Saha and Kveton(2023)}]{saha2023only}
\textsc{Saha, A.} and \textsc{Kveton, B.} (2023).
\newblock Only pay for what is uncertain: Variance-adaptive thompson sampling.
\newblock \textit{arXiv preprint arXiv:2303.09033} .

\bibitem[{Thompson(1933)}]{thompson1933likelihood}
\textsc{Thompson, W.~R.} (1933).
\newblock On the likelihood that one unknown probability exceeds another in view of the evidence of two samples.
\newblock \textit{Biometrika} \textbf{25} 285--294.

\bibitem[{Wang et~al.(2024{\natexlab{a}})Wang, Oertell, Agarwal, Kallus and Sun}]{wang2024more}
\textsc{Wang, K.}, \textsc{Oertell, O.}, \textsc{Agarwal, A.}, \textsc{Kallus, N.} and \textsc{Sun, W.} (2024{\natexlab{a}}).
\newblock More benefits of being distributional: Second-order bounds for reinforcement learning.
\newblock \textit{arXiv preprint arXiv:2402.07198} .

\bibitem[{Wang et~al.(2024{\natexlab{b}})Wang, Zhou, Lui and Sun}]{wang2024model}
\textsc{Wang, Z.}, \textsc{Zhou, D.}, \textsc{Lui, J.} and \textsc{Sun, W.} (2024{\natexlab{b}}).
\newblock Model-based rl as a minimalist approach to horizon-free and second-order bounds.
\newblock \textit{arXiv preprint arXiv:2408.08994} .

\bibitem[{Wei and Luo(2018)}]{wei2018more}
\textsc{Wei, C.-Y.} and \textsc{Luo, H.} (2018).
\newblock More adaptive algorithms for adversarial bandits.
\newblock In \textit{Conference On Learning Theory}. PMLR.

\bibitem[{Wei et~al.(2020)Wei, Luo and Agarwal}]{wei2020taking}
\textsc{Wei, C.-Y.}, \textsc{Luo, H.} and \textsc{Agarwal, A.} (2020).
\newblock Taking a hint: How to leverage loss predictors in contextual bandits?
\newblock In \textit{Conference on Learning Theory}. PMLR.

\bibitem[{Xu et~al.(2023)Xu, Min and Wang}]{xu2023noise}
\textsc{Xu, R.}, \textsc{Min, Y.} and \textsc{Wang, T.} (2023).
\newblock Noise-adaptive thompson sampling for linear contextual bandits.
\newblock \textit{Advances in Neural Information Processing Systems} \textbf{36} 23630--23657.

\bibitem[{Zhang(2022)}]{zhang2022feel}
\textsc{Zhang, T.} (2022).
\newblock Feel-good thompson sampling for contextual bandits and reinforcement learning.
\newblock \textit{SIAM Journal on Mathematics of Data Science} \textbf{4} 834--857.

\bibitem[{Zhang(2023)}]{zhang2023mathematical}
\textsc{Zhang, T.} (2023).
\newblock \textit{Mathematical analysis of machine learning algorithms}.
\newblock Cambridge University Press.

\bibitem[{Zhang et~al.(2021)Zhang, Yang, Ji and Du}]{zhang2021improved}
\textsc{Zhang, Z.}, \textsc{Yang, J.}, \textsc{Ji, X.} and \textsc{Du, S.~S.} (2021).
\newblock Improved variance-aware confidence sets for linear bandits and linear mixture mdp.
\newblock \textit{Advances in Neural Information Processing Systems} \textbf{34} 4342--4355.

\bibitem[{Zhao et~al.(2023)Zhao, He, Zhou, Zhang and Gu}]{zhao2023variance}
\textsc{Zhao, H.}, \textsc{He, J.}, \textsc{Zhou, D.}, \textsc{Zhang, T.} and \textsc{Gu, Q.} (2023).
\newblock Variance-dependent regret bounds for linear bandits and reinforcement learning: Adaptivity and computational efficiency.
\newblock In \textit{The Thirty Sixth Annual Conference on Learning Theory}. PMLR.

\bibitem[{Zhou and Gu(2022)}]{zhou2022computationally}
\textsc{Zhou, D.} and \textsc{Gu, Q.} (2022).
\newblock Computationally efficient horizon-free reinforcement learning for linear mixture mdps.
\newblock \textit{Advances in neural information processing systems} \textbf{35} 36337--36349.

\bibitem[{Zhou et~al.(2021)Zhou, Gu and Szepesvari}]{zhou2021nearly}
\textsc{Zhou, D.}, \textsc{Gu, Q.} and \textsc{Szepesvari, C.} (2021).
\newblock Nearly minimax optimal reinforcement learning for linear mixture markov decision processes.
\newblock In \textit{Conference on Learning Theory}. PMLR.

\end{thebibliography}
\clearpage
\appendix

\section{Generalized Decoupling Coefficient}\label{section:generalized_dc}

\subsection{Linear Reward Function Class}

In this section, we prove the decoupling coefficient for linear contextual bandits. We assume that $\btheta_*=\zero$ without loss of generality. We use the shorthand notation $\bphi_t\coloneqq\bphi(x_t, a_t)$, and define
\begin{align*}
\bSigma_t\coloneqq\lambda\Ib+\sum_{s=1}^t\beta_s\bphi_s\bphi_s^\top.
\end{align*}

The following lemma, known as the elliptical potential lemma, is a well-known result in previous works \citep{abbasi2011improved, zhou2022computationally}:
\begin{lemma}\label{lemma:elliptical_potential}
For any $\lambda>0$, we have
\begin{align*}
\sum_{t=1}^T\min\Big\{\beta_t\|\bphi_t\|_{\bSigma_{t-1}^{-1}}^2, 1\Big\}\le2(\log\det(\bSigma_T)-\log\det(\bSigma_0))\le2d\underbrace{\log(1+(\epsilon T)/(d\lambda))}_{\iota}.
\end{align*}
\end{lemma}
With Lemma \ref{lemma:elliptical_potential}, we decompose the sum of $\langle\btheta_t, \bphi_t\rangle$ into cases based on whether $\beta_t\|\bphi_t\|_{\bSigma_{t-1}^{-1}}^2$ is smaller than $1$, i.e.,
\begin{align}
\sum_{t=1}^T\langle\btheta_t, \bphi_t\rangle=\underbrace{\sum_{t=1}^T\langle\btheta_t, \bphi_t\rangle\ind\Big[\beta_t\|\bphi_t\|_{\bSigma_{t-1}^{-1}}^2\le1\Big]}_{I_1}+\underbrace{\sum_{t=1}^T\langle\btheta_t, \bphi_t\rangle\ind\Big[\beta_t\|\bphi_t\|_{\bSigma_{t-1}^{-1}}^2>1\Big]}_{I_2}.\label{eq:regret_ind_decomp}
\end{align}
For any $\gamma>0$, the term $I_1$ satisfies
\begin{align}
I_1&\le\sum_{t=1}^T\beta_t^{-1/2}\|\btheta_t\|_{\bSigma_{t-1}}\cdot\beta_t^{1/2}\|\bphi_t\|_{\bSigma_{t-1}^{-1}}\cdot\ind\Big[\beta_t\|\bphi_t\|_{\bSigma_{t-1}^{-1}}^2\le1\Big]\nonumber\\
&\le\sum_{t=1}^T\beta_t^{-1/2}\|\btheta_t\|_{\bSigma_{t-1}}\cdot\min\Big\{\beta_t^{1/2}\|\bphi_t\|_{\bSigma_{t-1}^{-1}}, 1\Big\}\nonumber\\
&\le\sum_{t=1}^T\frac\gamma{\beta_t}\|\btheta_t\|_{\bSigma_{t-1}}^2+\frac1{4\gamma}\sum_{t=1}^T\min\Big\{\beta_t\|\bphi_t\|_{\bSigma_{t-1}^{-1}}^2, 1\Big\}\nonumber\\
&\le\sum_{t=1}^T\frac\gamma{\beta_t}\bigg(\lambda\|\btheta_t\|_2^2+\sum_{s=1}^{t-1}\beta_s\langle\btheta_t, \bphi_s\rangle^2\bigg)+\frac{d\iota}{2\gamma}\nonumber\\
&\le\gamma\lambda\sum_{t=1}^T\frac1{\beta_t}+\sum_{t=1}^T\frac{\gamma}{\beta_t}\sum_{s=1}^{t-1}\beta_s\langle\btheta_t, \bphi_s\rangle^2+\frac{d\iota}{2\gamma},\label{eq:small_norm}
\end{align}
where the first inequality holds due to AM-GM inequality, the second inequality holds because $z\cdot\ind[z\le1]\le\min\{z, 1\}$, the third inequality holds due to Cauchy-Schwarz inequality, the fourth inequality holds due to Lemma \ref{lemma:elliptical_potential}, and the last inequality holds because $\|\btheta_t\|_2\le1$. The term $I_2$ satisfies
\begin{align}
I_2&\le\sum_{t=1}^T\ind\Big[\beta_t\|\bphi_t\|_{\bSigma_{t-1}^{-1}}^2>1\Big]\le\sum_{t=1}^T\min\Big\{\beta_t\|\bphi_t\|_{\bSigma_{t-1}^{-1}}^2, 1\Big\}\le2d\iota,\label{eq:large_norm}
\end{align}
where the first inequality holds because $\min\{|\btheta_t, \bphi_t|, 1\}\le1$, the second inequality holds because $\ind[z>1]\le\min\{z, 1\}$, and the last inequality holds due to Lemma \ref{lemma:elliptical_potential}. Plugging \eqref{eq:small_norm} and \eqref{eq:large_norm} into \eqref{eq:regret_ind_decomp}, we have
\begin{align*}
\sum_{t=1}^T\langle\btheta_t, \bphi_t\rangle\le\gamma\sum_{t=1}^T\sum_{s=1}^{t-1}\frac{\beta_s}{\beta_t}\langle\btheta_t, \bphi_s\rangle^2+\gamma\lambda\sum_{t=1}^T\frac1{\beta_t}+2d\iota\bigg(1+\frac1{4\gamma}\bigg).
\end{align*}

\subsection{General Reward Function Class}

In this section, we relate the generalized decoupling coefficient to the generalized Eluder dimension \citep{agarwal2023vo}. The proof is similar to that of the linear reward function. We first make the following decomposition:
\begin{align}
\sum_{t=1}^T(f_t(z_t)-f_*(z_t))&=\underbrace{\sum_{t=1}^T(f_t(z_t)-f_*(z_t))\ind\Big[\beta_t^{1/2}\cD_\cF(z_t; z_{[t-1]}, \bbeta_{[t-1]})\le 1\Big]}_{I_1}\nonumber\\
&\qquad+\underbrace{\sum_{t=1}^T(f_t(z_t)-f_*(z_t))\ind\Big[\beta_t\cD_\cF^2(z_t; z_{[t-1]}, \bbeta_{[t-1]})>1\Big]}_{I_2}.\label{eq:estimation_error_decomp}
\end{align}
The term $I_1$ satisfies
\begin{align}
I_1&=\sum_{t=1}^T\frac{\beta_t^{-1/2}(f_t(z)-f_*(z))}{\cD_{\cF}(z_t; z_{[t-1]}, \bbeta_{[t-1]})}\cdot\Big(\beta_t^{1/2}\cD_{\cF}(z_t; z_{[t-1]}, \bbeta_{[t-1]})\ind[\beta_t^{1/2}\cD_{\cF}(z_t; z_{[t-1]}, \bbeta_{[t-1]})\le1]\Big)\nonumber\\
&\le\sum_{t=1}^T\frac{\beta_t^{-1/2}(f_t(z_t)-f_*(z_t))}{\cD_{\cF}(z_t; z_{[t-1]}, \bbeta_{[t-1]})}\cdot\min\Big\{\beta_t^{1/2}\cD_{\cF}(z_t; z_{[t-1]}, \bbeta_{[t-1]}), 1\Big\}\nonumber\\
&\le\sum_{t=1}^T\frac{\gamma}{\beta_t}\cdot\frac{(f_t(z_t)-f_*(z_t))^2}{\cD_{\cF}^2(z_t, z_{[t-1]}; \bbeta_{[t-1]})}+\frac1{4\gamma}\sum_{t=1}^T\min\Big\{\beta_t\cD_{\cF}^2(z_t; z_{[t-1]}, \bbeta_{[t-1]}), 1\Big\}\nonumber\\
&\le\sum_{t=1}^T\frac{\gamma}{\beta_t}\bigg(\lambda+\sum_{s=1}^{t-1}\beta_s(f_t(z_s)-f_*(z_s))^2\bigg)+\frac{\dim(\cF, Z, \bbeta)}{4\gamma}\nonumber\\
&\le\sum_{t=1}^T\frac{\gamma}{\beta_t}\sum_{s=1}^{t-1}\beta_s(f_t(z_s)-f_*(z_s))^2+\gamma\lambda\sum_{t=1}^T\frac1{\beta_t}+\frac{\dim_{\epsilon, T}(\cF)}{4\gamma}\label{eq:small_D2}
\end{align}
where the first inequality holds because $z\ind[z\le 1]\le\min\{z, 1\}$, the second inequality holds due to AM-GM inequality, the third inequality holds due to the definition of $\cD_{\cF}^2(z_t, z_{[t-1]}, \bbeta_{[t-1]})$ and the definition of $\dim(\cF, Z, \bbeta)$, and the last inequality holds because $\dim_{\epsilon, T}(\cF)=\sup_{Z, \bbeta\le\epsilon}\dim(\cF, Z, \bbeta)$. The term $I_2$ satisfies
\begin{align}
I_2\le\sum_{t=1}^T\min\Big\{\beta_t\cD_{\cF}^2(z_t; z_{[t-1]}, \bbeta_{[t-1]}), 1\Big\}=\dim(\cF, Z, \bbeta)\le\dim_{\epsilon, T}(\cF),\label{eq:large_D2}
\end{align}
where the first inequality holds because $f_t(z_t)-f_*(z_t)\le1$, the equality holds due to the definition of $\dim(\cF, Z, \bbeta)$, and the last inequality holds because because $\dim_{\epsilon, T}(\cF)=\sup_{Z, \bbeta\le\epsilon}\dim(\cF, Z, \bbeta)$. Plugging \eqref{eq:small_D2} into \eqref{eq:large_D2}, we have
\begin{align*}
\sum_{t=1}^T(f_t(z_t)-f_*(z_t))\le\sum_{t=1}^T\frac{\gamma}{\beta_t}\sum_{s=1}^{t-1}\beta_s(f_t(z_s)-f_*(z_s))^2+\gamma\lambda\sum_{t=1}^T\frac1{\beta_t}+\bigg(1+\frac1{4\gamma}\bigg)\dim_{\epsilon, T}(\cF).
\end{align*}

\section{Proof of Main Theorem}\label{section:proof_variance_upper_bound}

In this section, we first provide a more general version of Theorem \ref{theorem:regret_upper_bound}:
\begin{theorem}\label{theorem:general}
Define
\begin{gather}
Z_t=-\EE_{S_{t-1}, x_t}\log\EE_{\tilde f\sim p_0}\exp\bigg(-\sum_{s=1}^{t-1}\bar\sigma_s^{-2}\Delta L(\tilde f, x_s, a_s, r_s)+\frac{c\sqrt{\Lambda_t}}{\bar\sigma_t^2}\FG_t(\tilde f)\bigg)\label{eq:def_Zt}\\
Z=1\vee\sup_{\{\bar\sigma_t\}_{t\in[T]}}\max_{t\in[T]}Z_t.\label{eq:def_Z}
\end{gather}
Suppose that the parameters are set as in \eqref{eq:parameters}, and the hyperparameters are
\begin{align*}
\lambda=1,\quad\epsilon=\alpha^{-2},\quad c=2\sqrt{Z/\dc_{\lambda, \epsilon, T}(\cF)}
\end{align*}
Then the total regret satisfies
\begin{align*}
&\EE[\Regret(T)]\le\frac94\sqrt{(\alpha^2T+\Lambda)Z\dc_{\lambda, \epsilon, T}(\cF)}+\dc_{\lambda, \epsilon, T}(\cF).
\end{align*}
\end{theorem}
To prove Theorem \ref{theorem:general}, we need the following lemma:
\begin{lemma}\label{lemma:main}
Under the conditions of Theorem \ref{theorem:general}, the following inequality holds:
\begin{gather*}
\EE_{S_{t-1}, x_t}\EE_{\tilde f\sim p_t}\bigg[\frac{\bar\sigma_t^2}{2c\sqrt{\Lambda}}\sum_{s=1}^{t-1}\bar\sigma_s^{-2}\LS_s(\tilde f)-\FG_t(\tilde f)\bigg]\le\frac{\bar\sigma_t^2}{c\sqrt{\Lambda_t}}Z_t.
\end{gather*}
\end{lemma}
We show the proof of Lemma \ref{lemma:main} in Appendix \ref{section:proof_lemmas}. We now provide the proof of Theorem \ref{theorem:general}:

\begin{proof}[Proof of Theorem \ref{theorem:general}]
The regret can be decomposed as
\begin{align}
&\Regret(t)=\sum_{t=1}^T[f_*(x_t, a_t^*)-f_*(x_t, a_t)]=\sum_{t=1}^T[f_t(x_t, a_t)-f_*(x_t, a_t)-\FG_t(f_t)]\nonumber\\
&\le\sum_{t=1}^T\bigg[\frac{\gamma}{\beta_t}\sum_{s=1}^{t-1}\beta_s\LS_s(f_t)-\FG_t(f_t)\bigg]+\gamma\lambda\sum_{t=1}^T\frac1{\beta_t}+\Big(1+\frac1{4\gamma}\Big)\dc_{\lambda, \epsilon, T}(\cF),\label{eq:regret_decomp}
\end{align}
where the equality holds due to the optimality of $a_t$, and the inequality holds due to the definition of the generalized decoupling coefficient. Taking the expectation on both sides of \eqref{eq:regret_decomp}, we have
\begin{align*}
&\EE[\Regret(T)]\le\sum_{t=1}^T\EE\bigg[\frac{\gamma}{\beta_t}\sum_{s=1}^{t-1}\beta_s\LS_s(f_t)-\FG_t(f_t)\bigg]+\gamma\lambda\sum_{t=1}^T\frac1{\beta_t}+\Big(1+\frac1{4\gamma}\Big)\dc_{\lambda, \epsilon, T}(\cF)\\
&=\sum_{t=1}^T\EE_{S_{t-1}, x_t}\EE_{\tilde f\sim p_t}\bigg[\frac{\gamma}{\beta_t}\sum_{s=1}^{t-1}\beta_s\LS_s(\tilde f)-\FG_t(\tilde f)\bigg]+\gamma\lambda\sum_{t=1}^T\frac1{\beta_t}+\Big(1+\frac1{4\gamma}\Big)\dc_{\lambda, \epsilon, T}(\cF),
\end{align*}
where the equality holds due to the double expectation theorem and because neither $\LS_s(f_t)$ nor $\FG_t(f_t)$ explicitly contain $a_t$. 
By selecting $\beta_t=\bar\sigma_t^{-2}$ and $\gamma=1/(2c\sqrt{\Lambda_T})$, we can use Lemma \ref{lemma:main} to further bound the regret:
\begin{align*}
\EE[\Regret(T)]&\le\frac{\lambda\sqrt{\Lambda_T}}{2c}+\Big(1+\frac{c\sqrt{\Lambda_T}}{2}\Big)\dc_{\lambda, \epsilon, T}(\cF)+\sum_{t=1}^T\frac{\bar\sigma_t^2}{c\sqrt{\Lambda_t}}Z_t\\
&\le\frac{\lambda\sqrt{\Lambda_T}}{2c}+\Big(1+\frac{c\sqrt{\Lambda_T}}{2}\Big)\dc_{\lambda, \epsilon, T}(\cF)+Z\sum_{t=1}^T\frac{\bar\sigma_t^2}{c\sqrt{\Lambda_t}}\\
&\le\frac{\lambda\sqrt{\Lambda_T}}{2c}+\Big(1+\frac{c\sqrt{\Lambda_T}}{2}\Big)\dc_{\lambda, \epsilon, T}(\cF)+2\frac{Z\sqrt{\Lambda_T}}{c}
\end{align*}
where the second inequality holds due to the definition of $Z$, and the last inequality holds due to Lemma \ref{lemma:sum_1/gamma_t}. Plugging in $c=2\sqrt{Z/\dc_{\lambda, \epsilon, T}(\cF)}$, we have
\begin{align*}
\EE[\Regret(T)]&\le\frac{\lambda}{4}\sqrt{\frac{\dc_{\lambda, \epsilon, T}(\cF)\Lambda_T}{Z}}+2\sqrt{Z\dc_{\lambda, \epsilon, T}(\cF)\Lambda_T}+\dc_{\lambda, \epsilon, T}(\cF)\\
&\le\frac94\sqrt{Z\dc_{\lambda, \epsilon, T}(\cF)\Lambda_T}+\dc_{\lambda, \epsilon, T}(\cF)\\
&\le\frac94\sqrt{(\alpha^2T+\Lambda)Z\dc_{\lambda, \epsilon, T}(\cF)}+\dc_{\lambda, \epsilon, T}(\cF),
\end{align*}
where the second inequality holds because $\lambda=1$ and $Z\ge1$, and the last inequality holds because $\Lambda_T=\sum_{t=1}^T\max\{\sigma_t^2, \alpha^2\}\le\sum_{t=1}^T(\sigma_t^2+\alpha^2)$.
\end{proof}

In order to prove Theorem \ref{theorem:regret_upper_bound} from Theorem \ref{theorem:general}, we note that $Z\le\log|\cF|$ using the argument in Section \ref{section:proof_sketch}. By selecting $\alpha=1/\sqrt{T}$, we can prove Theorem \ref{theorem:general}, and by selecting $\alpha$ to be a even smaller number, we can prove further shave the term $\alpha^2T+\Lambda$.

In order to deal with the infinite function class, we need the following lemma to characterize $Z$:
\begin{lemma}\label{lemma:Z}
Suppose that $\sigma_t$ is uniformly bounded by $\sigma$, and the hyperparameter $\alpha$ satisfies $\alpha\le1$. Define
\begin{align*}
\delta=\frac{\alpha^2}{2T}\min\{1, \sigma^{-1}\}.
\end{align*}
Then $Z$ satisfies
\begin{align*}
Z\le\log\frac1{p_0(\cF_\delta(f_*))}+\frac\alpha2+\frac{c}{\sqrt T}.
\end{align*}
For the linear reward function class where $\Theta$ is the unit ball, we have $-\log p_0(\cF_\delta(f_*))=\tilde\cO(d)$ according to \citet{zhang2022feel}. Therefore, the regret of \FGTSVA~in linear contextual bandits is $\tilde\cO(d\sqrt{\Lambda}+d)$, even if $\cF$ is not the $\epsilon$-net. \footnote{Although $Z_T$ has additional terms in $\alpha$ and $c$, they can be shaved by carefully choosing hypermeters in Theorem \ref{theorem:general}.}
\end{lemma}

\section{Proof of Lemmas in Appendix \ref{section:proof_variance_upper_bound}}

\subsection{Proof of Lemma \ref{lemma:main}}\label{section:proof_lemmas}
\begin{proof}
We aim to apply Lemma~\ref{lemma:Exp_exp=1}, so for any $s\in[t-1]$, we define $F_s(\tilde f, x_s, a_s, r_s)=\bar\sigma_s^{-2}\Delta L(\tilde f, x_s, a_s, r_s)$, then by Lemma \ref{lemma:Exp_exp=1}, we have
\begin{align}
0&=-\log\EE_{\tilde f\sim p_0}\EE_{S_{t-1}, x_t}\exp\bigg(\sum_{s=1}^{t-1}\xi_s(\tilde f, x_s, a_s, r_s)\bigg)\nonumber\\
&\le-\EE_{S_{t-1}, x_t}\log\EE_{\tilde f\sim p_0}\exp\bigg(\sum_{s=1}^{t-1}\xi_s(\tilde f, x_s, a_s, r_s)\bigg),\label{eq:jensen}
\end{align}
where the inequality holds due to the Jensen's inequality. We then use Lemma \ref{lemma:KL_reg} twice:
\begin{align}
&-\log\EE_{\tilde f\sim p_0}\exp\bigg(\sum_{s=1}^{t-1}\xi_s(\tilde f, x_s, a_s, r_s)\bigg)\le\EE_{\tilde f\sim p_t}\bigg[-\sum_{s=1}^{t-1}\xi_s(\tilde f, x_s, a_s, r_s)\bigg]+\KL(p_t||p_0),\label{eq:KL_ineq}\\
&-\log\EE_{\tilde f\sim p_0}\exp\bigg(-\sum_{s=1}^{t-1}\bar\sigma_s^{-2}\Delta L(\tilde f, x_s, a_s, r_s)+\lambda_t\FG_t(\tilde f)\bigg)\nonumber\\
&=\EE_{\tilde f\sim p_t}\bigg[\sum_{s=1}^{t-1}\bar\sigma_s^{-2}\Delta L(\tilde f, x_s, a_s, r_s)-\lambda_t\FG_t(\tilde f)\bigg]+\KL(p_t||p_0).\label{eq:KL_opt}
\end{align}
Furthermore, since
\begin{align*}
F_s(\tilde f, x_s, a_s, r_s)&=\bar\sigma_s^{-2}[(r_s-\tilde f(x_s, a_s))^2-(r_s-f_*(x_s, a_s))^2]\\
&=\bar\sigma_s^{-2}[(\epsilon_s+f_*(x_s, a_s)-\tilde f(x_s, a_s))^2-\epsilon_s^2]\\
&=\bar\sigma_s^{-2}\LS_s(\tilde f)-2\epsilon_s\bar\sigma_s^{-2}(\tilde f(x_s, a_s)-f_*(x_s, a_s)),
\end{align*}
we have
\begin{align}
&\log\EE[\exp(-F_s(\tilde f, x_s, a_s, r_s)|\cG_s)]\nonumber\\
&=-\bar\sigma_s^{-2}\Delta L_s(\tilde f)+\EE[\exp(2\epsilon_s\bar\sigma_s^{-2}(\tilde f(x_s, a_s)-f_*(x_s, a_s)))|\cG_t]\nonumber\\
&\le-\bar\sigma_s^{-2}\LS_s(\tilde f)+\sigma_s^2\bar\sigma_s^{-4}/2\LS_s(\tilde f)\nonumber\\
&\le-\bar\sigma_s^{-2}/2\cdot\LS_s(\tilde f),\label{eq:use_subgaussian}
\end{align}
where the first inequality holds due to Assumption \ref{assumption:subgaussian_noise}, and the second inequality holds because $\sigma_s\le\bar\sigma_s$. Plugging \eqref{eq:KL_ineq}, \eqref{eq:KL_opt}, and \eqref{eq:use_subgaussian} into \eqref{eq:jensen}, we have
\begin{align*}
0&\le\EE_{S_{t-1}, x_t}\bigg[\EE_{\tilde f\sim p_t}\bigg[-\sum_{s=1}^{t-1}\xi_s(\tilde f, x_s, a_s, r_s)\bigg]+\KL(p_t||p_0)\bigg]\\
&=\EE_{S_{t-1}, x_t}\bigg[\EE_{\tilde f\sim p_t}\bigg[\sum_{s=1}^{t-1}\Delta L(\tilde f, x_s, a_s, r_s)-\lambda_t\FG_t(\tilde f)\bigg]+\KL(p_t||p_0)\\
&\quad+\sum_{s=1}^{t-1}\EE_{\tilde f\sim p_t}\log\EE[\exp(-\Delta L(\tilde f, x_s, a_s, r_s))|\cG_t]+\lambda_t\EE_{\tilde f\sim p_t}\FG_t(\tilde f)\bigg]\\
&\le-\EE_{S_{t-1}, x_t}\log\EE_{\tilde f\sim p_0}\exp\bigg(-\sum_{s=1}^{t-1}\bar\sigma_s^{-2}\Delta L(\tilde f, x_s, a_s, r_s)+\lambda_t\FG_t(\tilde f)\bigg)\\
&\quad+\EE_{S_{t-1}, x_t}\EE_{\tilde f\sim p_t}\bigg[-\sum_{s=1}^{t-1}\frac{\bar\sigma_s^{-2}}2\LS_s(\tilde f)+\lambda_t\FG_t(\tilde f)\bigg].
\end{align*}
Rearranging terms and plugging in $\lambda_t=c\bar\sigma_t^{-2}\sqrt{\Lambda_t}$, we obtain
\begin{align*}
&\EE_{S_{t-1}, x_t}\EE_{\tilde f\sim p_t}\bigg[\frac{\bar\sigma_t^2}{2c\sqrt{\Lambda}}\sum_{s=1}^{t-1}\bar\sigma_s^{-2}\LS_s(\tilde f)-\FG_t(\tilde f)\bigg]\\
&\le\EE_{S_{t-1}, x_t}\EE_{\tilde f\sim p_t}\bigg[\frac{\bar\sigma_t^2}{2c\sqrt{\Lambda_t}}\sum_{s=1}^{t-1}\bar\sigma_s^{-2}\LS_s(\tilde f)-\FG_t(\tilde f)\bigg]\\
&\le\frac{-\bar\sigma_t^2}{c\sqrt{\Lambda_t}}\EE_{S_{t-1}, x_t}\log\EE_{\tilde f\sim p_0}\exp\bigg(-\sum_{s=1}^{t-1}\bar\sigma_s^{-2}\Delta L(\tilde f, x_s, a_s, r_s)+\frac{c\sqrt{\Lambda_t}}{\bar\sigma_t^2}\FG_t(\tilde f)\bigg).
\end{align*}
where the first inequality holds because $\Lambda_t\le\Lambda$.
\end{proof}

\subsection{Proof of Lemma \ref{lemma:Z}}

\begin{proof}[Proof of Lemma \ref{lemma:Z}]
We note that for any $f\in\cF_\delta(f_*)$, we have
\begin{align}
&-\sum_{s=1}^{t-1}\bar\sigma_s^{-2}\Delta L(f, x_s, a_s, r_s)+\frac{c\sqrt{\Lambda_t}}{\bar\sigma_t^2}\FG_t(f)\nonumber\\
&=\sum_{s=1}^{t-1}\bar\sigma_s^{-2}[2\epsilon_s(f(x_s, a_s)-f_*(x_s, a_s))-\LS_s(f)]+\frac{c\sqrt{\Lambda_t}}{\bar\sigma_t^2}\FG_t(f)\nonumber\\
&\ge-\sum_{s=1}^{t-1}\bar\sigma_s^{-2}(|\epsilon_s|\delta+\delta^2)-\frac{c\delta\sqrt{\Lambda_t}}{\bar\sigma_t^2}\label{eq:-base_lower_bound}
\end{align}
Therefore, $-Z_t$ can be lower bounded as
\begin{align*}
-Z_t&=\EE_{S_{t-1}, x_t}\log\EE_{\tilde f\sim p_0}\exp\bigg(-\sum_{s=1}^{t-1}\bar\sigma_s^{-2}\Delta L(\tilde f, x_s, a_s, r_s)+\frac{c\sqrt{\Lambda_t}}{\bar\sigma_t^2}\FG_t(\tilde f)\bigg)\\
&\ge\EE_{S_{t-1}, x_t}\log\bigg(p_0(\cF_\delta(f_*))\inf_{f\in\cF_\delta}\exp\bigg(-\sum_{s=1}^{t-1}\bar\sigma_s^{-2}\Delta L(f, x_s, a_s, r_s)+\frac{c\sqrt{\Lambda_t}}{\bar\sigma_t^2}\FG_t(f)\bigg)\bigg)\\
&\ge\log p_0(\cF_\delta(f_*))+\EE_{S_{t-1}, x_t}\bigg[-\sum_{s=1}^{t-1}\bar\sigma_s^{-2}(|\epsilon_s|\delta+\delta^2)-\frac{c\delta\sqrt{\Lambda_t}}{\bar\sigma_t^2}\bigg]\\
&\ge\log p_0(\cF_\delta(f_*))-\sum_{s=1}^{t-1}\Big(\frac{\delta}{2\bar\sigma_s}+\frac{\delta^2}{\bar\sigma_s^2}\Big)-\frac{c\delta\sqrt{\Lambda_t}}{\bar\sigma_t^2},
\end{align*}
where the first inequality holds because for any function $F(x)\ge0$, we have $\EE[F(x)]\ge\EE[F(x)\ind[x\in\cC]]\ge p(\cC)\EE[\inf_{x\in\cC}F(x)]$, the second inequality holds due to \eqref{eq:-base_lower_bound}, and the last inequality holds because $\EE|\epsilon_s|\le\sqrt{\EE\epsilon_s^2}\le\sqrt{\sigma_s^2/4}\le\bar\sigma_s/2$. By choosing $\delta=\alpha^2\min\{1/\sigma, 1\}/(2T)$, we have
\begin{align*}
Z_t\le\log\frac1{p_0(\cF_\delta(f_*))}+\frac{\alpha}{4}+\frac{\alpha^2}{4T}+\frac{c}{2\sqrt{T}}\le\log\log\frac1{p_0(\cF_\delta(f_*))}+\frac\alpha2+\frac{c}{\sqrt T},
\end{align*}
where the second inequality holds because $\alpha\le$ and $T\ge1$.
\end{proof}

\section{Auxiliary Lemmas}

\begin{lemma}\label{lemma:Exp_exp=1}
For any function $F_t: \cF\times\cX\times\cA\times\RR\to\RR$, define
\begin{align*}
\xi_t(\tilde f, x_t, a_t, r_t)&=-F_t(\tilde f, x_t, a_t, r_t)-\log\EE[\exp(-F_t(\tilde f, x_t, a_t, r_t))|\cG_t].
\end{align*}
Then for all $t$ and all $\tilde f\in\cF$, we have
\begin{align*}
\EE_{S_t}\exp\bigg(\sum_{s=1}^t\xi_s(\tilde f, x_s, a_s, r_s)\bigg)=1.
\end{align*}
\end{lemma}
\begin{proof}
We prove the lemma by induction. The property holds trivially for $t=0$. Now suppose that the lemma holds for $t-1$, then note that
\begin{align}
&\EE[\exp(\xi_t(\tilde f, x_t, a_t, r_t))|\cG_t]\nonumber\\
&=\EE\Big[\exp\Big(-F_t(\tilde f, x_t, a_t, r_t)-\log\EE[\exp(-F_t(\tilde f, x_t, a_t, r_t))|\cG_{t}]\Big)\Big|\cG_{t}\Big]\nonumber\\
&=\EE\bigg[\frac{\exp(-F_t(x_t, a_t, r_t))}{\EE[\exp(-F_t(\tilde f, x_t, a_t, r_t))|\cG_{t}]}\bigg|\cG_{t}\bigg]=1,\label{eq:condEE_exp_1}
\end{align}
where the first equality holds due to the definition of $\xi_t$, and the last equality holds because the denominator belongs to $\cG_t$. We then have
\begin{align*}
&\EE_{S_t}\exp\bigg(\sum_{s=1}^t\xi_s(\tilde f, x_s, a_s, r_s)\bigg)\\
&=\EE\bigg[\exp\bigg(\sum_{s=1}^{t-1}\xi_s(\tilde f, x_s, a_s, r_s)\bigg)\cdot\EE[\exp(\xi_t(\tilde f, x_s, a_s, r_s))|\cG_t]\bigg]\\
&=\EE\bigg[\exp\bigg(\sum_{s=1}^{t-1}\xi_s(\tilde f, x_s, a_s, r_s)\bigg)\bigg]=1
\end{align*}
where the first equality holds due to the double expectation theorem and because $\xi_s(\tilde s, x_s, a_s, r_s)\in\cG_t$ for all $s\in[t-1]$, the second equality holds due to \eqref{eq:condEE_exp_1}, and the last equality holds due to the induction hypothesis. Combining all the above, the lemma holds for all $t$ by induction.
\end{proof}

\begin{lemma}\label{lemma:sum_1/gamma_t}
Let $\bar\sigma_t$ and $\Lambda_t$ be defined in \eqref{eq:parameters}. Then the following property holds:
\begin{align*}
\sum_{t=1}^T\frac{\bar\sigma_t^2}{\sqrt{\Lambda_t}}\le2\sqrt{\Lambda_T}.
\end{align*}
\end{lemma}
\begin{proof}
We note that
\begin{align*}
\sqrt{\Lambda_T}&=\sum_{t=1}^T(\sqrt{\Lambda_t}-\sqrt{\Lambda_{t-1}})=\sum_{t=1}^T\frac{\Lambda_t-\Lambda_{t-1}}{\sqrt{\Lambda_t}+\sqrt{\Lambda_{t-1}}}\ge\sum_{t=1}^T\frac{\bar\sigma_t^2}{2\sqrt{\Lambda_t}},
\end{align*}
where the first equality holds due to the telescope sum and because $\Lambda_0=0$, and the inequality holds because $\Lambda_{t-1}\le\Lambda_t$ and $\bar\sigma_t^2=\Lambda_t-\Lambda_{t-1}$.
\end{proof}

\end{document}